\documentclass{article}

\usepackage[final]{neurips_2025}

\usepackage[utf8]{inputenc} % allow utf-8 input
\usepackage[T1]{fontenc}    % use 8-bit T1 fonts
\usepackage{hyperref}       % hyperlinks
\usepackage{url}            % simple URL typesetting
\usepackage{booktabs}       % professional-quality tables
\usepackage{amsfonts}       % blackboard math symbols
\usepackage{nicefrac}       % compact symbols for 1/2, etc.
\usepackage{microtype}      % microtypography
\usepackage{xcolor}         % colors
\usepackage{bm}
\usepackage{soul}
\usepackage{algorithm}
\usepackage{algorithmic}
\usepackage{amsmath}
\usepackage{booktabs}
\usepackage{mathtools}
\usepackage{natbib}
\usepackage{amssymb,amsmath,amsthm}
\usepackage{multirow}
\usepackage{arydshln}
\usepackage{enumitem}
\usepackage{comment}
\usepackage{wrapfig}
\usepackage{tikz}
\usepackage{adjustbox}
\usepackage{listings}
\usepackage{xcolor}
\definecolor{delim}{RGB}{20,105,176}
\definecolor{numb}{RGB}{106, 109, 32}
\definecolor{string}{rgb}{0.64,0.08,0.08}

\hypersetup{
    colorlinks=true,
    citecolor=blue,
}

\newtheorem{theorem}{Theorem}

\newtheorem{lemma}{Lemma}
\newtheorem{corollary}[theorem]{Corollary}

\newtheorem{assumption}{Assumption}
\newtheorem{remark}{Remark}

\DeclareMathOperator*{\argmax}{arg\,max}

\title{AdaDetectGPT: Adaptive Detection of LLM-Generated Text with Statistical Guarantees}

\author{%
  Hongyi~Zhou$^*$ \\
  Department of Mathematics\\
  Tsinghua University\\
  Beijing, China \\
  \And
  Jin~Zhu\thanks{Hongyi~Zhou and Jin~Zhu contributed equally to this paper and are listed in alphabetical order.} \\
  School of Mathematics \\
  University of Birmingham \\
  Birmingham, UK \\
  \And
  Pingfan Su \\
  Department of Statistics \\
  LSE  \\
  London, UK \\
  \And
  Kai Ye \\
  Department of Statistics \\
  LSE  \\
  London, UK \\
  \And
  Ying Yang \\
  Department of Statistics and Data Science\\
  Tsinghua University\\
  Beijing, China \\
  \And
  Shakeel A O B Gavioli-Akilagun$^{\dagger}$ \\
  Department of Decision Analytics and Operations \\
  City University Hong Kong \\
  Hongkong, China 
  \And
  Chengchun Shi\thanks{Corresponding authors: \texttt{sgavioli@cityu.edu.hk}, \texttt{c.shi7@lse.ac.uk}} \\
  Department of Statistics \\
  LSE  \\
  London, UK
}

\begin{document}

\maketitle

%Large language models (LLMs), such as the recently developed ChatGPT, can write documents, create executable code, and answer questions, often with human-like capabilities (Schulman et al., 2022). As these systems become more pervasive, there is increasing risk that they may be used for malicious purposes (Bergman et al., 2022; Mirsky et al., 2023). These include social engineering and election manipulation campaigns that exploit automated bots on social media platforms, creation of fake news and web content, and use of AI systems for cheating on academic writing and coding assignments. For many reasons, the ability to detect and audit the usage of machine-generated text becomes a key principle of harm reduction for large language models (Bender et al., 2021; Crothers et al., 2022; Grinbaum \& Adomaitis, 2022).

% - Use conditional curvature probability
% - Is it propper to use the  word power? 
% - call the witness function a witness function
% - use conditional log probability  
\begin{abstract}
We study the problem of determining whether a piece of text has been authored by a human or by a large language model (LLM). Existing state of the art logits-based detectors make use of statistics derived from the log-probability of the observed text evaluated using the distribution function of a given source LLM. However, relying solely on log probabilities can be sub-optimal. In response, we introduce AdaDetectGPT -- a novel classifier that adaptively learns a witness function from training data to enhance the performance of logits-based detectors. We provide statistical guarantees on its true positive rate, false positive rate, true negative rate and false negative rate. Extensive numerical studies show AdaDetectGPT nearly uniformly improves the state-of-the-art method in various combination of datasets and LLMs, and the improvement can reach up to 37\%. A python implementation of our method is available at \url{https://github.com/Mamba413/AdaDetectGPT}.
\end{abstract}

\section{Introduction}

Large language models (LLMs) such as ChatGPT \citep{openai2022chatgpt}, PaLM \citep{chowdhery2023palm}, Llama \citep{grattafiori2024llama} and DeepSeek \citep{bi2024deepseek} have revolutionized the field of generative artificial intelligence by enabling large-scale content generation across various fields including journalism, education, and creative writing \citep{demszky2023using, milano2023large, anil2024generativeAI}. However, their ability to produce highly human-like text poses serious risks, such as the spread of misinformation, academic dishonesty, and the erosion of trust in written communication \citep{ahmed2021detecting, lee2023plagiarize, christian2023cnet}. Consequently, accurately distinguishing between human- and LLM-generated text has emerged as a critical area of research. 

There is a growing literature on the detection of machine-generated text; refer to Section \ref{sec:relatedworks} for a review. One popular line of research focuses on statistics-based detectors, typically rely on log-probability outputs (i.e., logits) from a source LLM to construct the statistics for classification \citep[see e.g.,][]{gehrmann2019gltr,mitchell2023detectgpt}. These works are motivated by the empirical observation that LLM-generated text tends to exhibit higher log-probabilities or larger differences between the logits of original and perturbed tokens. However, as we demonstrated in Section \ref{sec:method}, relying solely on the logits can be sub-optimal for detecting LLM-generated text.

\textbf{Our contribution}. In this paper, we propose AdaDetectGPT (see Figure \ref{fig:detectgpt} for a visualization), an adaptive LLM detector that leverages external training data to enhance the effectiveness of existing logits-based detectors. Our approach derives a lower bound on the true negative rate (TNR) of logits-based detectors and adaptively learns a witness function by optimizing this bound, resulting in a more powerful detection statistic. The optimization is straightfoward and requires only solving a system of linear equations. Based on this statistic, we further introduce an approach to select the classification threshold that controls AdaDetectGPT's false negative rate (FNR).

Empirically, we conduct extensive evaluations across multiple datasets and a variety of target language models to demonstrate that AdaDetectGPT consistently outperforms existing logits-based detectors. In white-box settings -- where the target LLM to be detected is the same as the source LLM used to compute the logits -- AdaDetectGPT achieves improvements over the best alternative in area under the curve (AUC) ranging from 12.5\% to 37\%. In black-box settings, where the source and target LLMs differ, it similarly offer gains of up to 20\%.

Theoretically, we provide statistical performance guarantees for AdaDetectGPT, deriving finite-sample error bounds for its TNR, FNR, true positive rate (TPR) and false positive rate (FPR). Existing literature on logits-based detectors generally lacks systematic statistical analysis.  Our work aims to fill in this gap and contribute toward a deeper understanding of these methods in this emerging field, by offering a comprehensive analysis based on the aforementioned standard classification metrics. 

\begin{figure}
    \vspace*{-17pt}
    \centering
    \includegraphics[width=\linewidth]{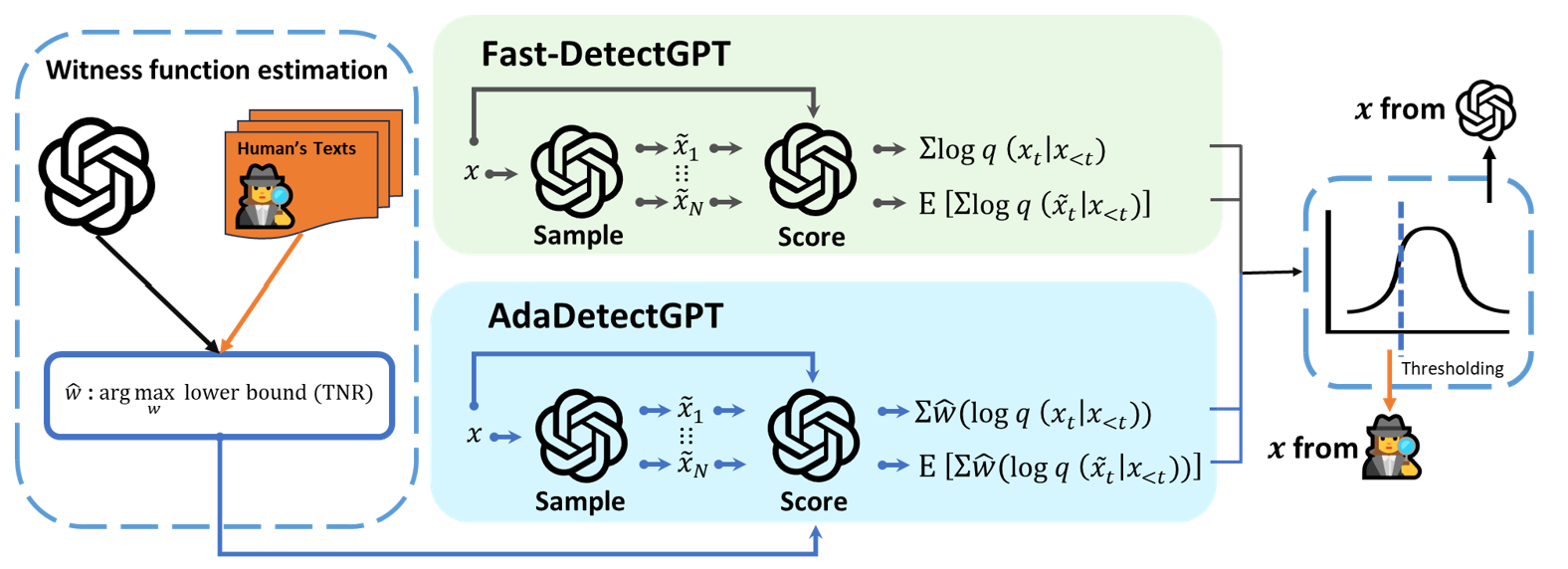}
    \vspace*{-12pt}
    \caption{Workflow of AdaDetectGPT. Built upon Fast-DetectGPT \citep{bao2024fastdetectgpt}, our method adaptively learn a witness function $\widehat{w}$ from training data by maximizing a lower bound on the TNR, while using normal approximation for FNR control. }
    \label{fig:detectgpt}
\end{figure}

\subsection{Related works}\label{sec:relatedworks}
Existing methods for detecting machine-generated text generally fall into three categories: machine learning (ML)-based, statistics-based, and watermarking-based; see \citet{yang2024survey,wu2025survey} for recent comprehensive surveys. Our method is most closely related to the first two categories and, unlike the third, does not rely on knowledge of the specific hash function or random number generator used during token generation, which are often model-specific and not publicly available. In what follows, we review the first two categories and defer the discussion of watermarking-based approaches to Appendix~\ref{sec:related-work-watermarking}. 

\textbf{ML-based detection}. ML-based methods train classification models on external human- and machine-authored text for detection. Many methods can be further categorized into two types. The first type extracts certain features from text and apply classical ML models to train classifiers based on these features. Various features have been proposed in the literature, ranging from classical term frequency-inverse document frequency (TF-IDF), unigram, and bigram features \citep{solaiman2019release}, to more complex features engineered specifically for this task, such as the cross-entropy loss computed between the source text and a surrogate LLM \citep{guo2024biscope} and the rewriting-based measure that quantifies the difference between original texts and their LLM-rewritten versions \citep{mao2024raidar}. 

The second type of methods fine-tune LLMs directly for classification. This approach is intuitive, as LLMs are inherently designed for processing text data; we only need to modify the model’s output to predict a binary label rather than token probabilities. Various LLMs have been employed for fine-tuning, including RoBERTa  \citep{solaiman2019release,guo2023close}, BERT \citep{ippolito2020bert} and DistilBERT \citep{mitrovic2023chatgpt}. 

In addition to these two types of methods, \citet{abburi2023generative} propose a hybrid approach that uses the outputs of fine-tuned LLMs as input features for classical ML-based classification. Further efforts have focused on handling adversarial attacks \citep{crothers2022adversarial, krishna2023paraphrasing, koike2024outfox, sadasivan2025can}, short texts \citep{tian2024multiscale}, out-of-distribution texts \citep{guo2024detective}, unobserved prompts \citep{zhou2026learn}, biases against non-native English writers \citep{liang2023gpt}, accommodating statistical inference \citep{zhou2026detecting}, as well as the downstream applications of these methods in domains such as education, social media and medicine \citep{herbold2023large,kumarage2023stylometric,liao2023differentiating}.

\textbf{Statistics-based detection}. Statistics-based methods leverage differences in token-level metrics such as log-probabilities to distinguish between human- and machine-authored text. Unlike ML-based approaches, many of these methods do not rely on external training data; instead, they directly use predefined statistical measures as classifiers.  
In particular, a seminal work by \citet{gehrmann2019gltr} propose several such measures, including the average log-probability and the distribution of absolute ranks of probabilities of tokens across a text. %with both probabilities derived from an LLM's output logits. 
These measures exhibit substantial differences between human- and machine-authored text, and they have been widely employed and extended in the literature \citep[see e.g.,][]{mitchell2023detectgpt, su2023detectllm, bao2024fastdetectgpt, hans2024spotting}. 

Other statistical measures employed are calculated based on the N-gram distributions \citep{solaiman2019release,yang2024dnagpt}, the intrinsic dimensionality of text \citep{tulchinskii2023intrinsic}, the reward model used in LLMs \citep{lee2024remodetect} and the maximum mean discrepancy  \citep{zhang2024detecting, song2025deep}. Recent works have extended these methods to more challenging scenarios, such as to handle adversarial attacks \citep{hu2023radar}, machine-revised text \citep{chen2025imitate} and black-box settings \citep{yu2024dpic,zeng2024dlad}. Theoretically, \citet{chakraborty2024position} establish a sample complexity bound for detecting machine-generated text. 

To conclude this section, we remark that our proposal lies at the intersection between statistics- and ML-based methods. Similar to many statistics-based approaches, our classifier is constructed based on the log-probabilities. However, we adaptively learn a witness function via ML to improve its effectiveness. In this way, our method leverages the strengths of both approaches, leading to superior detection performance. %Furthermore, we introduce a theoretically guaranteed algorithm for threshold selection in our classifier -- a problem that has been less studied in prior statistics-based work. 

\section{Preliminaries}\label{sec:preliminary}
We first define the white-box and black-box settings as well as our objective. We next review two baseline methods, DetectGPT \citep{mitchell2023detectgpt} and Fast-DetectGPT \citep{bao2024fastdetectgpt}, as they are closely related to our proposal. Finally, we introduce the martingale central limit theorem \citep[see e.g.,][]{hall2014martingale}, which serves as the theoretical basis for our threshold selection. 

\textbf{Task and settings}. We study the problem of determining whether a given passage $\bm{X}$, represented as a sequence of tokens $(X_1, X_2, \ldots, X_L)$, was authored by a human or generated by a target LLM. Specifically, let $p$ and $q$ denote the distributions over human-written and LLM-generated tokens, respectively. Each distribution can be represented as a product of conditional probability functions, $\prod_t p_t$ for humans and $\prod_t q_t$ for the target LLM, where each \( p_t(x_{t}| x_{< t}) \) (and similarly \( q_t \)) denotes the conditional probability mass function of the next token \( x_{t} \) given the preceding tokens, where \( x_{< t}\coloneqq(x_1, x_2, \ldots, x_{t-1}) \) when $t > 1$ and $ x_{< t} \coloneqq \emptyset$ otherwise.  Our goal is to develop a classifier to discriminate between $\bm{X} \sim p$ (human) and $\bm{X} \sim q$ (LLM). 

We assume access to a source LLM's probability distribution function $q'=\prod_t q_t'$. When $q'=q$, it corresponds to the \textit{white-box setting} where the source model we have is the same as the target model we wish to detect. This is the primary setting considered in this paper. For closed-source LLMs such as GPT-3.5 and GPT-4, their probability functions are not publicly available. In such cases, we utilize an open-source model with distribution $q'$ as an approximation of $q$, resulting in the \textit{black-box setting}, which our method is also extended to handle.

We also assume access to a corpus of $n$ human-authored passages $\mathcal{H}=\{\bm{X}^{(i)}\}_{i=1}^n\sim p$. %and another set of passages $\mathcal{M}=\{\bm{M}^{(1)}\}\sim q'$ generated by the source LLM. 
This assumption is reasonable, as large corpora of human-written text are readily available online (e.g., Wikipedia). %and LLM-generated passages can be obtained by rewriting these human texts. 
Without loss of generality, we assume that all passages have the same number of tokens $L$, achieved by zero-padding shorter sequences to match the maximum token length. Throughout this paper, we use boldface letters (e.g., $\bm{X}$) to denote passages, and non-boldface letters (e.g., $X$) to denote individual tokens. 

\textbf{Baseline methods}. Both DetectGPT and Fast-DetectGPT are statistics-based and rely on the log-probability of a passage, $\log q'(\bm{X})$, as the basis for classification. Specifically, DetectGPT considers the following statistic:
\begin{equation}\label{eqn:detectstatistics}
    \frac{\log q'(\bm{X})-\mathbb{E}_{\widetilde{\bm{X}}\sim p'(\bullet|\bm{X})}[\log q'(\widetilde{\bm{X}})] }{\sqrt{\textrm{Var}_{\widetilde{\bm{X}}\sim p'(\bullet|\bm{X})}(\log q'(\widetilde{\bm{X}}))}},
\end{equation}
where both the expectation in the numerator and the variance in the denominator are evaluated under a perturbation function $p'$, which produces $\widetilde{\bm{X}}$ that is a slightly modified version of $\bm{X}$ with similar meaning. The rationale behind this statistic is that, empirically, machine-generated text tends to yield higher values than human-written text when evaluated using \eqref{eqn:detectstatistics} \citep[][Figure 2]{mitchell2023detectgpt}. As a result, a passage is classified as machine-generated if this statistic is larger than a certain threshold. 

A potential limitation of DetectGPT is that sampling from the perturbation distribution $p'$ requires multiple calls to the source LLM to generate rewritten versions of the input passage, making the calculation of \eqref{eqn:detectstatistics} computationally expensive. Fast-DetectGPT addresses this issue by proposing a modified version of \eqref{eqn:detectstatistics}, given by 
\begin{eqnarray}\label{eqn:fastdetectstatistics}
    \frac{\sum_t \log  q_t'(X_t| X_{<t})-\sum_t \mathbb{E}_{\widetilde{X}_t \sim s_t(\bullet | X_{<t})} \log q_t'(\widetilde{X}_t | X_{<t})}{\sqrt{\sum_t \textrm{Var}_{\widetilde{X}_t \sim s_t(\bullet | X_{<t})} (\log q_t'(\widetilde{X}_t | X_{<t}))}}.
\end{eqnarray}
Specifically, notice that \( \log q'(\bm{X}) \) can be decomposed as a token-wise sum \( \sum_t \log q_s(X_t | X_{<t}) \). Thus, the first term in the numerator of \eqref{eqn:detectstatistics} is the same as that in \eqref{eqn:fastdetectstatistics}. However, Fast-DetectGPT replaces the centering term in \eqref{eqn:detectstatistics} with $\sum_t \mathbb{E}_{\widetilde{X}_t \sim s_t(\bullet | X_{<t})} \log q_t'(\widetilde{X}_t | X_{<t})$. Here, $s=\prod_t s_t$ denotes a sampling distribution function which may equal $q$ or be derived from another LLM. By replacing the perturbation function with $s$, the centering term can be efficiently computed directly from the LLM's conditional probabilities. Additionally, due to the conditioning on $X_{<t}$ in the centering term, the variance term is equal to the sum of the conditional variances of \( \log q_t'(\widetilde{X}_t| X_{<t}) \). Finally, it classifies a passage as machine-generated if this modified statistic is larger than a certain threshold.  

\textbf{Martingale central limit theorem}. The martingale central limit theorem (MCLT) is a fundamental result in probability theory that enables rigorous statistical inference for time-dependent data. It is well-suited for analyzing text data where tokens are generated sequentially given their predecessors. Consider a time series $\{Z_t\}_{t}$ where each $Z_t$ represents a real-valued random variable. Suppose there exists a sequence of monotonically increasing sets of random variables $\mathcal{F}_1\subseteq \mathcal{F}_2 \subseteq \cdots$ so that $Z_t\in \mathcal{F}_t$ for any $t$. Under certain regularity conditions, MCLT states that the normalized partial sum 
\begin{equation}\label{eqn:partialsum}
    \frac{\sum_{t=1}^L [Z_t-\mathbb{E}(Z_t|\mathcal{F}_{t-1})]}{\sqrt{\sum_{t=1}^L\textrm{Var}(Z_t|\mathcal{F}_{t-1})}}
\end{equation}
converges in distribution to a standard normal random variable as $L$ approaches infinity \citep{brown1971martingale}. Notice that the statistic in \eqref{eqn:partialsum} shares similar structures with that employed in Fast-DetectGPT (see \eqref{eqn:fastdetectstatistics}). In the next section, we will leverage this connection for FNR control through normal approximation.

\section{AdaDetectGPT}\label{sec:method}
We present AdaDetectGPT in this section. We begin by discussing the white-box setting in Parts~\hyperlink{method-a}{(a)}--\hyperlink{method-c}{(c)}: Part~\hyperlink{method-a}{(a)} introduces the proposed statistical measure; Part \hyperlink{method-b}{(b)} discusses our choice of the classification threshold for FNR control; Part \hyperlink{method-c}{(c)} derives a lower bound on the TNR to learn the witness function. Next, in Part \hyperlink{method-d}{(d)}, we extend our proposal to the black-box setting. Finally, in Part \hyperlink{method-e}{(e)}, we establish the statistical properties of AdaDetectGPT. 
%A pseudocode summarizing our procedure is provided in Algorithm \ref{algo:stat4gpt}. 

\begin{comment}
\begin{algorithm}[t]
    \caption{AdaDetectGPT}\label{algo:stat4gpt}
    \begin{algorithmic}[1]
        \REQUIRE text $X$ with $L$ tokens, LLM $p_\theta$, nominal significance level $\alpha$, estimated witness function $\widehat{w}$, estimated shift \\
        \ENSURE Reject null hypothesis -- machine-generated, Accept null hypothesis -- Human-written 
        \FOR{$j=1, \ldots, L$}
        \STATE Compute $y_j$, $\mu_j$ and $\sigma_j$ according to Equation~\eqref{eq:statistics-components}. 
        \ENDFOR
        \STATE Compute the likelihood-based statistics $T$ according to~\eqref{eq:classification-statistics}. 
        \STATE If $|T| > z_{1-\alpha}$, Reject the null hypothesis $\mathcal{H}_0$ ; otherwise, accept null hypothesis. 
    \end{algorithmic}
\end{algorithm}
\end{comment}

\hypertarget{method-a}{\textbf{(a) Statistical measure}}. Notice that $q'=q$ under the white-box setting. Given a passage $\bm{X}$, the proposed classifier is built upon the following statistic:
\begin{equation}\label{eqn:Tw}
    T_w(\bm{X}) \coloneqq\frac{\sum_t [w(\log  q_t(X_t| X_{<t}))-\mathbb{E}_{\widetilde{X}_t \sim q_t(\bullet | X_{<t})} w(\log q_t(\widetilde{X}_t | X_{<t}))]}{\sqrt{\sum_t \textrm{Var}_{\widetilde{X}_t \sim q_t(\bullet | X_{<t})} (w(\log q_t(\widetilde{X}_t | X_{<t})))}},
\end{equation}
where $w:\mathbb{R}\to \mathbb{R}$ denotes a one-dimensional witness function defined over the space of log-probabilities.

By definition, $T_w(\bm{X})$ is very similar to Fast-DetectGPT's statistic in \eqref{eqn:fastdetectstatistics}, with two modifications: (i) First, rather than using the raw log conditional probability $q_t$, we apply a witness function $w$ to these $\log q_t$ to enhance the detection power of the resulting classifier. Our numerical experiments demonstrate that this transformation better distinguishes between human- and machine-authored text (see Figure~\ref{fig:separate}). Below, we further provide a simple analytical example to illustrate this power enhancement. (ii) Second, we set the the sampling function $s$ in \eqref{eqn:fastdetectstatistics} to the source LLM's $q$. This allows the numerator to match the form of the partial sum in \eqref{eqn:partialsum}, which enables the application of MCLT for FNR control. 

\begin{figure}[t]
    % \vspace*{-20pt}
    \centering
    \includegraphics[width=0.49\linewidth]{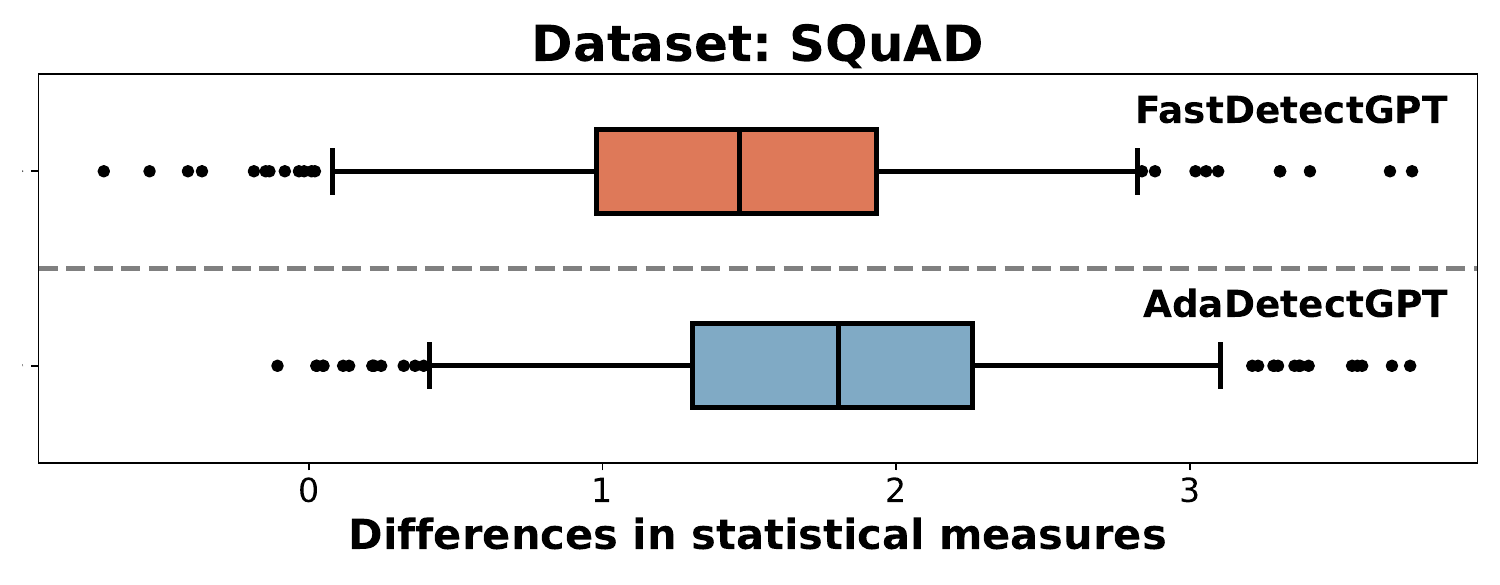}
    \includegraphics[width=0.5\linewidth]{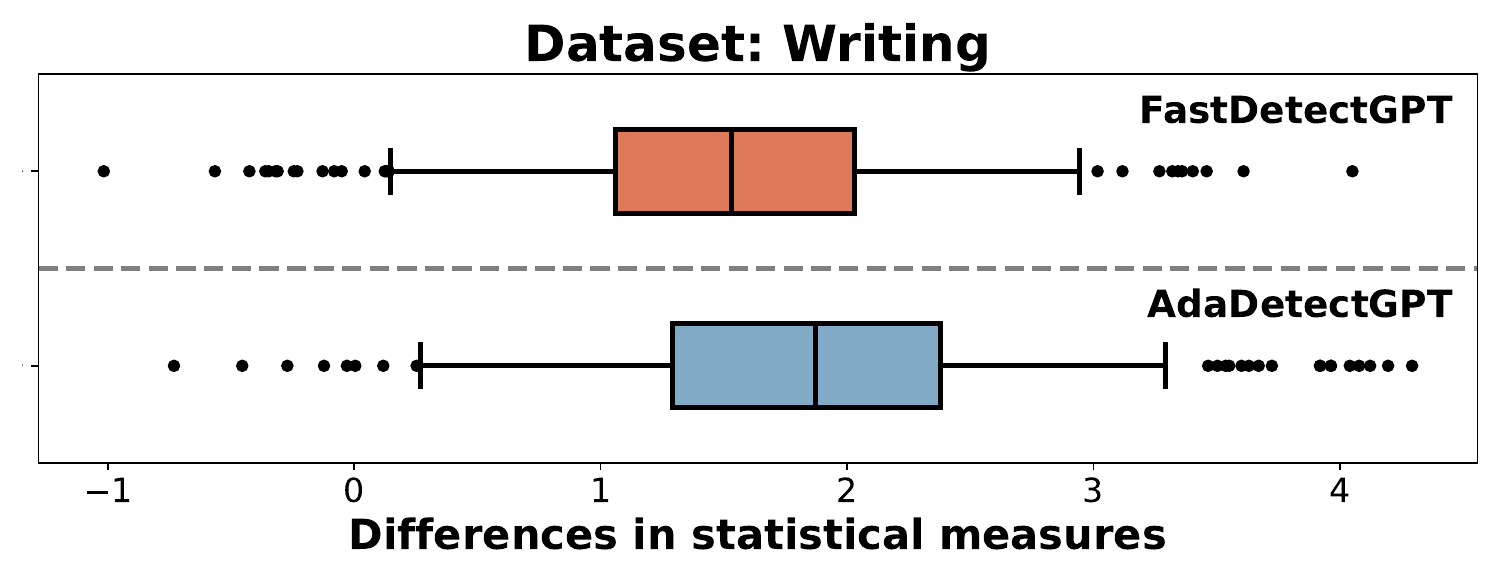}
    \vspace*{-10pt}
    \caption{Boxplots visualizing the differences in the statistical measures between human- and LLM-authored passages, comparing AdaDetectGPT (with a learned witness function) and Fast-DetectGPT (without it). A larger positive difference from zero indicates better detection power. As observed, the difference computed by AdaDetectGPT is consistently larger than that of Fast-DetectGPT across the first quartile, median, and third quartile. The left panel shows statistics evaluated on the SQuAD dataset, while the right panel displays results for the WritingPrompts dataset.}
    \label{fig:separate}
\end{figure}

\ul{\textit{An analytical example}}. Consider a hypothetical ``Kingdom of Bit'' where all communication uses just two tokens: 1 (yes) and 0 (no). Let $\{p_t\}_t$ denote the true token distribution of this language. Suppose a malicious wizard creates synthetic citizens who appear similar to ordinary people, but their language follows a simpler distribution $q_t(x_t|x_{<t})=q(x_t)$ for some fixed function $q$, being independent of the prior context $x_{<t}$. As we will show later, the detection power of our statistic crucially depends on the following quantity:
\begin{equation}\label{eqn:examplestatistics}
    \frac{1}{L}\sum_{t=1}^L \big[ \mathbb{E}_{\widetilde{X}_t \sim q} w(\log q(\widetilde{X}_t)) -   \mathbb{E}_{X_t \sim p}w(\log q(X_t)) \big].
\end{equation}
The greater the deviation of the expression~\eqref{eqn:examplestatistics} from zero, the higher the power to detect synthetic humans. When setting $w$ to the identity function, this expression reduces to
\begin{align*}
    \log\left( \frac{q(1)}{q(0)} \right) \times \left[ q(1) - \frac{1}{L}\sum_{t=1}^L  p_t(1) \right],
\end{align*} 
where $p_t(1)=\mathbb{E}_{X_{<t} \sim p} p(1|X_{<t})$ is determined by the true language distribution $p$. In this case, \eqref{eqn:examplestatistics} converges to zero as $q(1)\to 1/2$, regardless of the difference between $q(1)$ and the average of $p_t(1)$. As such, there are simple settings in which existing logits-based detectors with an identity witness function will struggle to distinguish between human and machine authored text, independent of the actual distance between $p$ and $q$. However, for any $q(1) \neq 1/2$, there exists a function $w$ that makes \eqref{eqn:examplestatistics} equal to $q(1) - L^{-1}\sum_{t=1}^L [p_t(1)]$, independent of the log ratio (see Appendix \ref{section: discussion about witness function} for a formal proof). Thus, whenever $q(1)$ differs from the average $p_t(1)$, an appropriate transformation $w$ can reliably detect synthetic humans. 

We will discuss how to properly learn the witness function below. Since this is inherently a classification problem, it is natural to learn a witness function that maximizes the AUC of the resulting detector. Equivalently, this amounts to finding a witness function that maximizes the TNR for each fixed FNR. In \hyperlink{method-b}{(b)}, we discuss how to select the classification threshold to control the FNR at a specified level. Given this threshold, we then derive a lower bound on the TNR in \hyperlink{method-c}{(c)} to be maximized for learning the witness function.

% As such, for existing logits-based detectors with an identity witness function, we expect they are generally challenging to detect synthetic humans when $q(1)$ is close to $1/2$.

\hypertarget{method-b}{\textbf{(b) Classification threshold}}. Given a witness function $w$, we aim to determine a threshold $c$ so that the FNR of the classifier $T_w(\bm{X})>c$ is below a specified level $\alpha>0$. A key observation is that, when the passage $\bm{X}$ is generated from the source LLM, $T_w(\bm{X})$ in \eqref{eqn:Tw} can be represented by the partial sum in \eqref{eqn:partialsum} with $Z_t=w(\log  q_t(X_t| X_{<t}))$ and $\mathcal{F}_t=X_{<t}$. It follows from the MCLT that 
\begin{eqnarray}\label{eqn:FNR}
    \textrm{FNR}_w\coloneqq\mathbb{P}_{\bm{X}\sim q}(T_w(\bm{X})\le c)\to \Phi(c),\quad \textrm{as}~L\to \infty, 
\end{eqnarray}
where $\Phi(\bullet)$ denotes the cumulative distribution function of a standard normal random variable. Thus, to ensure the desired FNR control, we set the threshold $c$ to the $\alpha$th quantile of $\Phi$, denoted by $z_{\alpha}$. 

We validate this threshold selection both theoretically and empirically. Specifically, Theorem~\ref{thm:FNR} in \hyperlink{method-e}{(e)} establishes a finite-sample error bound for our classifier's FNR. Figure~\ref{fig:tpr-fpr-normal} illustrates the effectiveness of FNR control and normal approximation across three benchmark datasets and two language models.

\begin{figure}[t]
    \centering
    % \vspace*{-0.24in}
    % \includegraphics[width=1.0\linewidth]{figure/TPR_FPR_normal.pdf}
    % \includegraphics[width=1.0\linewidth]{figure/TPR_FPR_normal_v1.pdf}
    \includegraphics[width=1.0\linewidth]{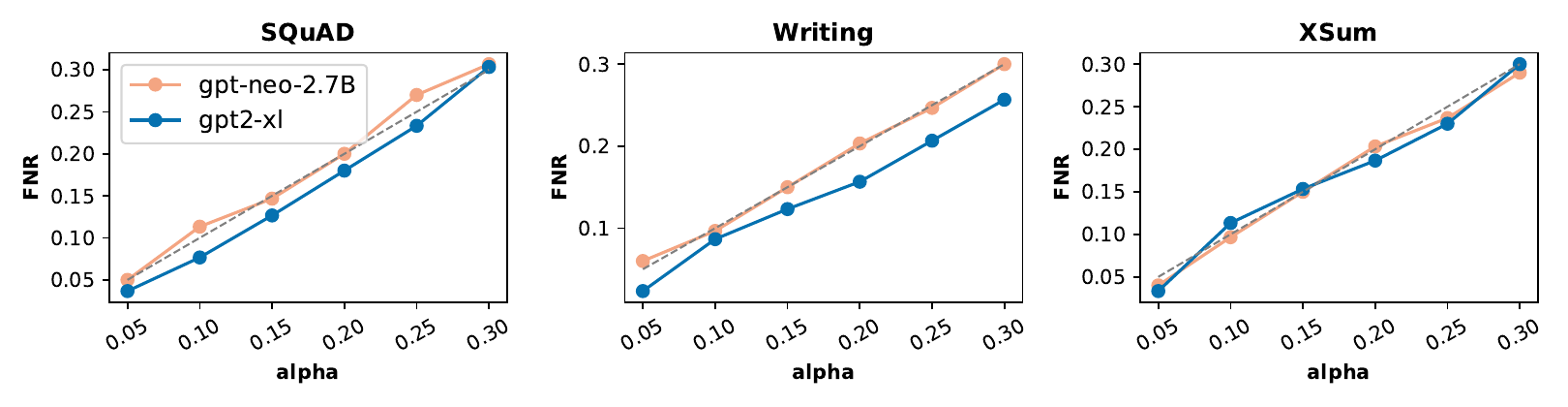}
    \vspace{-0.15in}
    \includegraphics[width=1.0\linewidth]{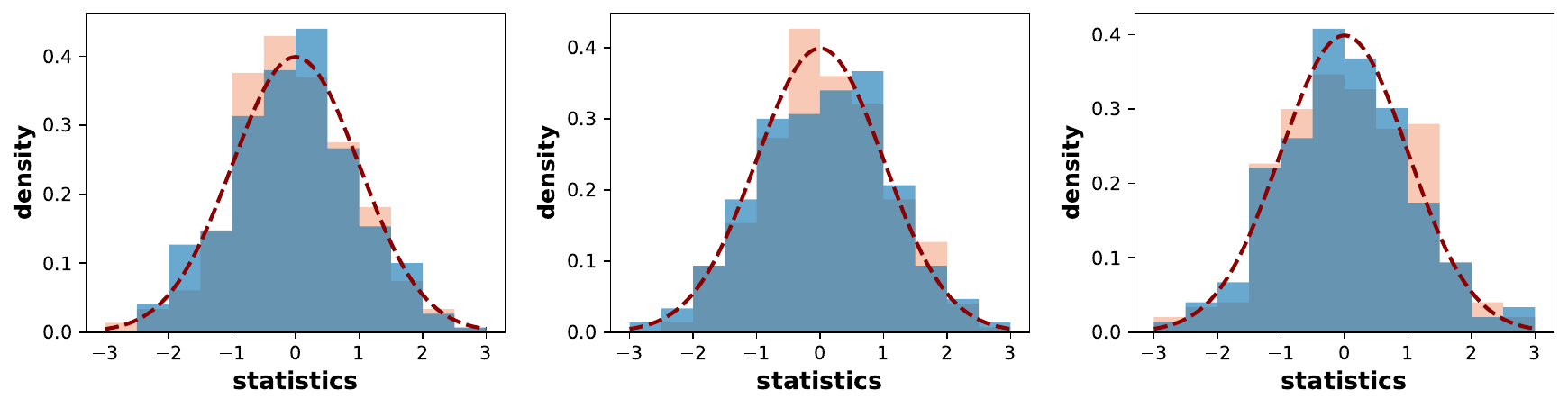}
    % \vspace*{-10pt}
    \caption{Top panel: FNR of the classifier plotted against the significance level $\alpha$. Bottom panel: Distribution of statistics evaluated on LLM-generated text. The dashed red line is the density function of standard normal random variable. Results are shown across three datasets (from left to right) and two language models (indicated by different colors). }
    \label{fig:tpr-fpr-normal}
\end{figure}

\hypertarget{method-c}{\textbf{(c) Learning the witness function}}. With the classifier's FNR fixed asymptotically at level $\alpha$, one can identify the optimal witness function $w$ by maximizing its TNR, defined by 
\begin{equation}\label{eqn:TNR}
    \textrm{TNR}_w\coloneqq\mathbb{P}_{\bm{X}\sim p}(T_w(\bm{X})\le z_{\alpha}),
\end{equation}
where the probability is evaluated under the human-generated text distribution $p$. Toward that end, we employ MCLT again to derive a closed-form expression of \eqref{eqn:TNR}.  

By definition, we can represent $T_w(\bm{X})$ by the difference $T_w^{(1)}(\bm{X})-T_w^{(2)}(\bm{X})$ where 
\begin{eqnarray*}
   T_w^{(1)}(\bm{X}) &=&\frac{\sum_{t=1}^T [\log w(X_t|X_{<t})-\mathbb{E}_{\widetilde{X}_t \sim p_t} \log w(\widetilde{X}_t|X_{<t})]}{\sqrt{\sum_t \textrm{Var}_{\widetilde{X}_t \sim q_t} (\log w(\widetilde{X}_t|X_{<t})))}}, \nonumber\\
    T_w^{(2)}(\bm{X})&=&\frac{\sum_t [\mathbb{E}_{\widetilde{X}_t \sim q_t} w(\log q_t(\widetilde{X}_t|X_{<t}))-\mathbb{E}_{\widetilde{X}_t \sim p_t} w(\log q_t(\widetilde{X}_t|X_{<t}))]}{\sqrt{\sum_t \textrm{Var}_{\widetilde{X}_t \sim q_t} (w(\log q_t(\widetilde{X}_t|X_{<t})))}}.
\end{eqnarray*}
Here, to ease notations, $p_t$ and $q_t$ in the expectation and variance are implicitly taken with respect to their conditional distributions $p_t(\bullet | X_{<t})$ and $q_t(\bullet | X_{<t})$ given $X_{<t}$.

Notice that $T_w^{(1)}$ is very similar to $T_w$ --- the only difference lies in the centering term: the conditional expectation is taken with respect to $p_t$ instead of $q_t$. Under an equal variance condition where the variance terms in the denominator evaluated at $q_t$ and $p_t$ converge to the same quantity asymptotically, we can invoke the MCLT to prove the asymptotic normality of $T_w^{(1)}(\bm{X})$ under the human distribution $p$. As a result, $\textrm{TNR}_w$ can be approximated by
\begin{eqnarray}\label{eqn:TNRapp}
    \mathbb{P}_{\bm{X}\sim p}(T_w^{(1)}(\bm{X})\le z_{\alpha}+T_w^{(2)}(\bm{X}))\approx \mathbb{E}_{\bm{X}\sim p} \Phi(z_{\alpha}+T_w^{(2)}(\bm{X})).
\end{eqnarray}
Given the external human-written text dataset $\mathcal{H}$, one can estimate the right-hand-side of \eqref{eqn:TNRapp} by $n^{-1} \sum_{i=1}^n \Phi(z_{\alpha}+T_w^{(2)}(\bm{X}^{(i)}))$ and optimize this estimator to compute the witness function $\widehat{w}$. However, the resulting $\widehat{w}$ depends on the choice of $\alpha$, which limits its flexibility. To eliminate this dependence, we derive a lower bound on the TNR in the following theorem. 
\begin{theorem}[TNR lower bound]\label{thm:TNRlowerbound}
    Under the equal variance condition specified in Appendix \ref{sec:assumptions}, $\textrm{TNR}_w$ is asymptotically lower bounded by $\min\{\alpha+\Phi'(z_{\alpha})T_w^{(2*)}, 1-\alpha \}$ where $\Phi'$ is the derivative of $\Phi$ and $T_w^{(2*)}$ denotes a population version of $T_w^{(2)}(\bm{X})$, given by
    \begin{eqnarray}
        T_w^{(2*)}=\frac{\sum_t [\mathbb{E}_{X_{<t}\sim p, \widetilde{X}_t \sim q_t} w(\log q_t(\widetilde{X}_t|X_{<t}))-\mathbb{E}_{X_{<t}\sim p, \widetilde{X}_t \sim p_t} w(\log q_t(\widetilde{X}_t|X_{<t}))] }{\sqrt{\sum_t \mathbb{E}_{X_{<t}\sim p} \textrm{Var}_{\widetilde{X}_t \sim q_t} (w(\log q_t(\widetilde{X}_t|X_{<t})))}}.
        \label{equation: T_w^*}
    \end{eqnarray}
\end{theorem}
Compared to $T_w^{(2)}(\bm{X})$, both the numerator and denominator of $T_w^{(2*)}$ are defined by taking expectations with respect to $X_{<t}\sim p$ for each $t$. In the analytical example, the numerator simplifies to \eqref{eqn:examplestatistics} after appropriate scaling. According to Theorem \ref{thm:TNRlowerbound}, optimizing the lower bound is equivalent to maximizing $T_w^{(2*)}$, whose solution is independent of $\alpha$. We also remark that the maximal value $\max_w T_w^{(2*)}$ is similar to certain integral probability metrics \citep{muller1997integral} such as the maximum mean discrepancy measure widely studied in machine learning \citep[see e.g.,][]{gretton2012kernel}. 

Motivated by Theorem \ref{thm:TNRlowerbound}, we replace the expectations $\mathbb{E}_{X_{<t}\sim p}$ in both the numerator and denominator of $T_w^{(2*)}$ with their empirical average over the dataset $\mathcal{H}$, denote the resulting estimator by $\widehat{T}_w^{(2)}$ and compute $\widehat{w}=\arg\max_{w\in \mathcal{W}} \widehat{T}_w^{(2)}$ over a function class $\mathcal{W}$. Since $w$ is a one-dimensional function over the space of real-valued logits, the optimization is relatively simple. 

Specifically, we adopt a linear function class $\mathcal{W}=\{w(z)=\phi(z)^\top\beta: \| \beta \|_2 = 1 \}$ for some bounded $d$-dimensional feature mapping $\phi$. In this case, $\widehat{T}_w^{(2)}$ simplifies to $\beta^\top \psi/\sqrt{\beta^\top \Sigma \beta}$ for some $d$-dimensional vector $\psi$ and $d\times d$ semi-definite positive matrix $\Sigma$ (see Appendix \ref{sec:w-estimation} for the detailed derivation). With some calculations, our estimated regression coefficients $\widehat{\beta}$ can be efficiently obtained by solving the linear system $\Sigma \beta=\psi$, leading to $\widehat{w}= \widehat{\beta}^\top \phi$. We set $\phi$ to the B-spline basis function \citep{de1978practical} in our implementation and relegate additional details to Appendix \ref{sec:w-estimation}.

\hypertarget{method-d}{\textbf{(d) Extension to black-box settings}}. When the target LLM’s logits are unavailable, we employ an open-source LLM with a distribution similar to that of the target model to construct the statistical measure in \eqref{eqn:Tw}. The witness function is then learned in the same manner as described in \hyperlink{method-c}{(c)}. We empirically evaluate this approach in Section~\ref{sec:experiment}.
%$q_1,q_2,\ldots,q_m$, rather than a single source model, to enhance the detection power of AdaDetectGPT. Specifically, for each source distribution $q_i$, we apply the procedure in \hyperlink{method-c}{(c)} to learn a corresponding witness function $\widehat{w}_i$, resulting in the statistic $T_{\widehat{w}_i}(\bm{X})$. We next take the maximum over all source models and use $\max_i T_{\widehat{w}_i}(\bm{X})$ for LLM detection.

\begin{wrapfigure}[10]{r}{0.28\textwidth}
    \vspace{-0.39in}
    \includegraphics[width=\linewidth]{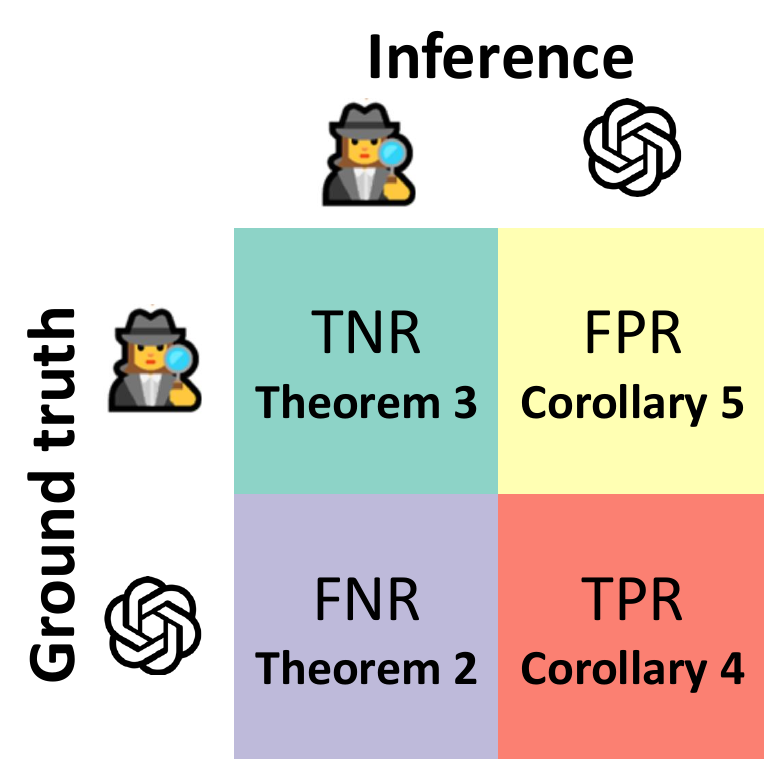}
    % \vspace{-0.28in}
    \vspace{-0.15in}
    \caption{A summary of our theories.}\label{fig:theory}
    \vspace{-0.13in}
\end{wrapfigure}
\hypertarget{method-e}{\textbf{(e) Statistical guarantees}}. In the following, we establish finite-sample bounds for the FNR, TNR, FPR and TPR of the proposed classifier in the white-box setting; see Figure \ref{fig:theory} for a summary of our theories. Recall that $L$ denotes the number of tokens in a passage, $n$ denotes the number of human-authored passages in the training data and $\alpha$ is the target FNR level we wish to control. We also define $w^*=\arg\max_{w\in \mathcal{W}} T_w^{(2*)}$ as the population limit of our estimated witness function $\widehat{w}$. Finally, let $V_L$ denote the ratio $$\frac{L^{-1}\sum_{t=1}^L \textrm{Var}_{\widetilde{X}_t \sim q_t}(\widehat{w}(\log q_t(\widetilde{X}_t|X_{<t})))}{L^{-1}\sum_{t=1}^L \textrm{Var}_{\bm{X}\sim q}(\widehat{w}(\log q_t(X_t|X_{<t})))}.$$

\begin{theorem}[FNR]\label{thm:FNR}
    Assume the denominator of $V_L$ is bounded away from zero. Then the expected FNR of our classifier FNR$_{\widehat{w}}$ is upper bounded by $ \mathbb{E}(\textrm{FNR}_{\widehat{w}})\le \alpha + O\big(L^{-1/2}\log L\big)+O\big( \mathbb{E} |V_L-1|^{1/3} \big)$, where the expectation on the left-hand-side is taken with respect to $\widehat{w}$. 
\end{theorem}

\begin{theorem}[TNR]\label{thm:TNR}
    Under a minimal eigenvalue assumption in Appendix \ref{sec:assumptions}, with some absolute constant $\gamma > 0$ depending only on the eigenvalue decay, the expected TNR of our classifier TNR$_{\widehat{w}}$ is lower bounded by $ \mathbb{E} (\textrm{TNR}_{\widehat{w}})\ge \textrm{TNR}_{w^*}- O \big ( d^\gamma / \sqrt{n} \big ). $
\end{theorem}
\begin{corollary}[TPR]\label{thm:TPR}
    Under the condition in Theorem \ref{thm:FNR}, the expected TPR of our classifier %TPR$_{\widehat{w}}=1-\textrm{FNR}_{\widehat{w}}$ 
    is lower bounded by $\mathbb{E} (\textrm{TPR}_{\widehat{w}})=1-\mathbb{E} (\textrm{FNR}_{\widehat{w}})\ge 1- \alpha - O\big(L^{-1/2}\log L \big)-O\big(\mathbb{E} |V_L-1|^{1/3}\big)$.
\end{corollary}
\begin{corollary}[FPR]\label{thm:FPR}
    Under the condition in Theorem \ref{thm:TNR}, the expected FPR of our classifier is upper bounded by $\mathbb{E}(\textrm{FPR}_{\widehat{w}})= 1-\mathbb{E}(\textrm{TNR}_{\widehat{w}})\le 1 - \text{TNR}_{w^*} + O \big ( d^\gamma / \sqrt{n} \big )$. 
\end{corollary}
We make a few remarks: 
\begin{itemize}[leftmargin=*]
    \item Theorem \ref{thm:FNR} provides an upper bound on the difference between our classifier's expected FNR and the target level $\alpha$. This excess FNR depends on two factors: (i) the number of tokens $L$, and (ii) the variance ratio $V_L$. As $L$ diverges to infinity and $V_L$ stabilizes to $1$, the FNR of our classifier approaches the target level. We remark that assumptions similar to the bounded denominator condition in Theorem \ref{thm:FNR} are commonly imposed \citep[see e.g.,][]{Bolthausen_1982_exact,hall2014martingale}. 
    \item Theorem \ref{thm:TNR} establishes an upper bound on the difference between our classifier's TNR and that of the oracle classifier, which has access to the population-level optimal witness function $w^*$ . This difference depends on: (i) the training sample size $n$, and (ii) the dimensionality $d$ of the feature mapping $\phi$. Since $\phi$ is defined over the space of one-dimensional log probabilities of tokens, rather than the token space itself, $d$ is independent of the vocabulary size, and can be treated as fixed. Consequently, as the training data size $n$ grows to infinity, our classifier's TNR converges to that of the oracle classifier. According to Theorem \ref{thm:TNRlowerbound}, with a sufficiently large $\max\limits_{w\in \mathcal{W}} T_w^{(2*)}$, the expected TNR can reach up to $1-\alpha$. 
    \item Corollaries \ref{thm:TPR} and \ref{thm:FPR} follow directly from Theorems \ref{thm:FNR} and \ref{thm:TNR}, due to the relationships between TPR and FNR, and between FPR and TNR.
\end{itemize}
Together,  these results suggest that our classifier is expected to achieve a high AUC as the significance level $\alpha$ varies, which we will verify empirically in the next section.

\section{Experiments}\label{sec:experiment}
We conduct numerical experiments in both white-box and black-box settings to illustrate the usefulness of AdaDetectGPT. %Due to space constraints, %results under some additional setups are presented in Appendix~\ref{sec:addblackbox}, where our improvement can reach up to 30\%. 
To save space, some implementation details are provided in Appendix~\ref{sec:experiment-detail}. 
%\subsection{Experiments in white-box settings}

%\textbf{White-box settings}. 
\textbf{Datasets}. We consider five widely-used datasets for comparing different detectors, including \ul{\textit{SQuAD}} for Wikipedia-style question answering \citep{rajpurkar2016squad}, \ul{\textit{WritingPrompts}} for story generation \citep{fan2018hierarchical}, \ul{\textit{XSum}} for news summarization \citep{Narayan2018DontGM}, \ul{\textit{Yelp}} for crowd-sourced product reviews \citep{zhangCharacterlevelConvolutionalNetworks2015}, and \ul{\textit{Essay}} for high school and university-level essays \citep{verma2024ghostbuster}. Following \citet{bao2024fastdetectgpt}, we randomly sample 500 human-written paragraphs from each dataset and generate an equal number of machine-authored paragraphs by prompting an LLM with the first 120 tokens of the human-written text and requiring it to complete the text with up to 200 tokens. This is a challenging setting where LLM-generated text is mixed with human writing. To evaluate AdaDetectGPT, we compute the AUC on each of the five datasets, with its witness function $\widehat{w}$  trained on two randomly selected datasets that differ from the test dataset.

\textbf{Benchmark methods.} In white-box settings, we compare the proposed AdaDetectGPT against \textbf{eight} state-of-the-art detectors: \ul{\textit{Likelihood}}, \ul{\textit{Entropy}}, \ul{\textit{LogRank}} \citep{gehrmann2019gltr}, LogRank Ratio \citep[\ul{\textit{LRR}},][]{su2023detectllm}, \ul{\textit{DetectGPT}} \citep{mitchell2023detectgpt} and its variants Normalized Perturbed log Rank  \citep[\ul{\textit{NPR}},][]{su2023detectllm},  \ul{\textit{Fast-DetectGPT}} \citep{bao2024fastdetectgpt}, \ul{\textit{DNAGPT}} \citep{yang2024dnagpt}. In black-box settings, we further compare against \ul{\textit{RoBERTaBase}} and \ul{\textit{RoBERTaLarge}} \citep{solaiman2019release}, \ul{\textit{Binoculars}} \citep{hans2024spotting}, \ul{\textit{RADAR}} \citep{hu2023radar},  
and \ul{\textit{BiScope}} \citep{guo2024biscope}, but omit \textit{DetectGPT}, \textit{NPR} and \textit{DNAGPT} due to their high computational cost. This yields \textbf{ten} baseline algorithms. We measure the detection power of each detector using AUC. %Detailed experimental settings are summarized in Appendix~\ref{sec:experiment-detail}.
 
\textbf{White-box results.} We first consider the \ul{\textit{white-box setting}} where we employ various source models summarized in Table~\ref{tab:llms} for text generation, with parameter sizes ranging from 1 to 20 billion. Table~\ref{tab:main_results_details} reports the AUC scores of various detectors across all combinations of datasets and five source models. It can be seen that AdaDetectGPT achieves the highest AUC across all combinations of datasets and source models, outperforming Fast-DetectGPT -- the best baseline method -- by 12.5\%-37\%. We also evaluate AdaDetectGPT on three more advanced open-source LLMs: Qwen2.5 \citep{bai2025qwen2}, Mistral \citep{albert2024mistral}, and LLaMA3 \citep{grattafiori2024llama}. As shown in Table~\ref{tab:white-box-advanced}, AdaDetectGPT delivers consistent improvements over Fast-DetectGPT and maintains competitive performance across five datasets, achieving the best results in most cases. These findings highlight the advantage of using an adaptively learned witness function for classification. 

In Appendix~\ref{sec:computational-analysis}, we analyze the \ul{\textit{computational cost}} of AdaDetectGPT. Training the witness function typically requires less than one minute across different sample sizes and dimensions, with memory usage below 0.5 GB. In Appendix~\ref{sec:tuning}, we further conduct a \ul{\textit{sensitivity analysis}} to investigate the sensitivity of AdaDetectGPT's AUC score to various factors that may affect the estimation of the witness function. Our results show that AdaDetectGPT is generally robust to the size of training data, the number of B-spline features, and the distributional shift between training and test data, consistently maintaining superior performance over baselines.

\begin{table}[t]
    % \vspace{-0.24in}
    \caption{AUC scores of various detectors in the white-box setting. The ``Relative'' rows report the percentage improvement of AdaDetectGPT over Fast-DetectGPT.}
    \label{tab:main_results_details}
%    \vspace{-0.2in}
    \centering\small
    \begin{tabular}{llcccccc}
        \toprule
        \multirow{2}{*}{\bf Dataset} & \multirow{2}{*}{\bf Method} & \multicolumn{6}{c}{\bf Source Model} \\
         &  & GPT-2 & OPT-2.7 & GPT-Neo & GPT-J & GPT-NeoX & Avg. \\    
        \midrule
        \multirow{10}{*}{SQuAD} 
        & Likelihood & 0.7043 & 0.6872 & 0.6722 & 0.6414 & 0.5948 & 0.6600 \\
        & Entropy & 0.5513 & 0.5289 & 0.5392 & 0.5385 & 0.5402 & 0.5396 \\
        & LogRank & 0.7328 & 0.7169 & 0.7047 & 0.6738 & 0.6179 & 0.6892 \\
        & LRR & 0.7703 & 0.7623 & 0.7553 & 0.7286 & 0.6632 & 0.5849 \\
        & NPR & 0.8783 & 0.8342 & 0.8311 & 0.7330 & 0.6454 & 0.6182 \\
        & DNAGPT & 0.8640 & 0.8413 & 0.8106 & 0.7671 & 0.6933 & 0.7953 \\
        & DetectGPT & 0.8565 & 0.8167 & 0.8047 & 0.7124 & 0.6380 & 0.6047 \\
        \cdashline{2-8}
        & Fast-DetectGPT & 0.9042 & 0.8801 & 0.8731 & 0.8417 & 0.7866 & 0.8571 \\
        & AdaDetectGPT & \textbf{0.9265} & \textbf{0.9058} & \textbf{0.9088} & \textbf{0.8618} & \textbf{0.8292} & \textbf{0.8864} \\
        \cdashline{2-8}
        % & Diff & 0.0223 & 0.0258 & 0.0357 & 0.0200 & 0.0425 & 0.0293 \\
        & Relative (\parbox[c]{1em}{\tikz{\hspace*{3.9pt}\draw[blue, -latex] (0,0) -- (0,0.3);}}) & 23.2696 & 21.4914 & 28.1306 & 12.6659 & 19.9396 & 20.4909 \\
        \midrule
        \multirow{10}{*}{Writing} 
        & Likelihood & 0.7297 & 0.7119 & 0.7032 & 0.6937 & 0.6811 & 0.7039 \\
        & Entropy & 0.5154 & 0.5132 & 0.5098 & 0.5126 & 0.5275 & 0.5157 \\
        & LogRank & 0.7510 & 0.7332 & 0.7253 & 0.7141 & 0.7029 & 0.7253 \\
        & LRR & 0.7801 & 0.7599 & 0.7624 & 0.7489 & 0.7362 & 0.6015 \\
        & NPR & 0.8760 & 0.8329 & 0.8373 & 0.8008 & 0.7854 & 0.4942 \\
        & DNAGPT & 0.8962 & 0.8503 & 0.8590 & 0.8450 & 0.8199 & 0.8541 \\
        & DetectGPT & 0.8540 & 0.8060 & 0.8173 & 0.7681 & 0.7514 & 0.6286 \\
        \cdashline{2-8}
        & Fast-DetectGPT & 0.8972 & 0.8891 & 0.8904 & 0.8879 & 0.8782 & 0.8886 \\
        & AdaDetectGPT & \textbf{0.9352} & \textbf{0.9175} & \textbf{0.9290} & \textbf{0.9190} & \textbf{0.9158} & \textbf{0.9233} \\
        \cdashline{2-8}
        % & Diff & 0.0379 & 0.0284 & 0.0386 & 0.0311 & 0.0376 & 0.0347 \\
        & Relative (\parbox[c]{1em}{\tikz{\hspace*{3.9pt}\draw[blue, -latex] (0,0) -- (0,0.3);}}) & 36.9089 & 25.5699 & 35.2136 & 27.7142 & 30.8963 & 31.1543 \\
        \midrule
        \multirow{10}{*}{XSum} 
        & Likelihood & 0.6410 & 0.6264 & 0.6197 & 0.6086 & 0.5906 & 0.6173 \\
        & Entropy & 0.5637 & 0.5571 & 0.5864 & 0.5606 & 0.5679 & 0.5671 \\
        & LogRank & 0.6623 & 0.6493 & 0.6493 & 0.6310 & 0.6109 & 0.6406 \\
        & LRR & 0.6952 & 0.6795 & 0.6967 & 0.6612 & 0.6447 & 0.5361 \\
        & NPR & 0.8211 & 0.7713 & 0.8273 & 0.7676 & 0.7290 & 0.6178 \\
        & DNAGPT & 0.7531 & 0.7283 & 0.7160 & 0.6939 & 0.6634 & 0.7109 \\
        & DetectGPT & 0.8073 & 0.7639 & 0.8244 & 0.7599 & 0.7239 & 0.6110 \\
        \cdashline{2-8}
        & Fast-DetectGPT & 0.8293 & 0.8067 & 0.8137 & 0.7926 & 0.7622 & 0.8009 \\
        & AdaDetectGPT & \textbf{0.8534} & \textbf{0.8420} & \textbf{0.8532} & \textbf{0.8347} & \textbf{0.8061} & \textbf{0.8379} \\
        \cdashline{2-8}
        % & Diff & 0.0241 & 0.0353 & 0.0395 & 0.0422 & 0.0439 & 0.0370 \\
        & Relative (\parbox[c]{1em}{\tikz{\hspace*{3.9pt}\draw[blue, -latex] (0,0) -- (0,0.3);}}) & 14.1250 & 18.2659 & 21.1936 & 20.3301 & 18.4492 & 18.5779 \\
        \bottomrule
    \end{tabular}
    % \vspace{-0.1in}
\end{table}

\textbf{Black-box results}. We next consider the \ul{\textit{black-box setting}}, where the task is to detect text generated by three widely used advanced LLMs: GPT-4o \citep{hurst2024gpt}, Claude-3.5-Haiku \citep{anthropic2024claude}, and Gemini-2.5-Flash \citep{comanici2025gemini}. In this setting, token-level log-probabilities are not publicly accessible. To implement Fast-DetectGPT and AdaDetectGPT, we use \texttt{google/gemma-2-9b} and \texttt{google/gemma-2-9b-it} \citep{team2024gemma} as the sampling and scoring models, respectively, to construct the classification statistics described in Section~\ref{sec:preliminary}. %We further benchmark against existing ML-based classifiers, including GPT-2 detectors based on RoBERTa-base and RoBERTa-large provided by OpenAI. 
The results are reported in Tables~\ref{tab:advanced-gpt4o-claude} and~\ref{tab:advanced-gemini}. Overall, Fast-DetectGPT remains the strongest baseline, although it occasionally underperforms Binoculars or RADAR. Nonetheless, AdaDetectGPT consistently improves upon Fast-DetectGPT across various datasets, with gains on Essay reaching up to 37.8\%.  

Additionally, we evaluate AdaDetectGPT's robustness to two adversarial attacks, paraphrasing and decoherence, in the black-box setting. As shown in Table~\ref{tab:adversarial-attack}, AdaDetectGPT demonstrates greater resilience than Fast-DetectGPT to adversarially perturbed texts. The improvement reaches up to 10\% for paraphrasing and up to 85\% for decoherence. Finally, we employ the same five LLMs from Table \ref{tab:main_results_details} to compare Fast-DetectGPT and AdaDetectGPT in black-box settings. Following \cite{bao2024fastdetectgpt}, we use GPT-J as the source LLM for detecting each of the remaining four target LLMs. %Additionally, we implement a variant of AdaDetectGPT that uses GPT-J for scoring and GPT-Neo for sampling. 
Due to space constraints, the results are presented in Table~\ref{tab:main_results_blackbox_details} of Appendix~\ref{sec:addblackbox}, where AdaDetectGPT uniformly outperforms Fast-DetectGPT in all cases, with improvements of up to 29\%.

\begin{table}[t]
\vspace*{-24pt}
\centering
\caption{AUC scores of various detectors to detect text generated by GPT-4o and Claude-3.5 across datasets.}\label{tab:advanced-gpt4o-claude}
\begin{adjustbox}{width=1.0\textwidth}
\setlength{\tabcolsep}{3pt}
\begin{tabular}{l|ccccc|ccccc}
\toprule
\multirow{2}{*}{\textbf{Method}} & \multicolumn{5}{|c|}{\textbf{GPT-4o}} & \multicolumn{5}{c}{\textbf{Claude-3.5}} \\
\cline{2-11}
& XSum & Writing & Yelp & Essay & Avg. & XSum & Writing & Yelp & Essay & Avg. \\
\midrule
RoBERTaBase   & 0.5141 & 0.5352 & 0.6029 & 0.5739 & 0.5655 & 0.5206 & 0.5386 & 0.5630 & 0.5593 & 0.5566 \\
RoBERTaLarge  & 0.5074 & 0.5827 & 0.5027 & 0.6575 & 0.5626 & 0.5462 & 0.6149 & 0.5105 & 0.6063 & 0.5695 \\
Likelihood    & 0.5194 & 0.7661 & 0.8425 & 0.7849 & 0.7282 & 0.6780 & 0.7502 & 0.7134 & 0.6160 & 0.6894 \\
Entropy       & 0.5397 & 0.7021 & 0.7291 & 0.6951 & 0.6665 & 0.5935 & 0.6594 & 0.6625 & 0.5742 & 0.6142 \\
LogRank       & 0.5123 & 0.7478 & 0.8259 & 0.7786 & 0.7161 & 0.5109 & 0.6653 & 0.7244 & 0.7111 & 0.6529 \\
LRR           & 0.5116 & 0.6099 & 0.6828 & 0.6930 & 0.6243 & 0.5268 & 0.5560 & 0.5527 & 0.6535 & 0.5723 \\
Binoculars    & 0.9022 & 0.9572	& 0.9840 & 0.9777 & 0.9552 & 0.9012 & 0.9393 & 0.9752 & 0.9603 & 0.9440 \\
RADAR         & \textbf{0.9580} & 0.8046 & 0.8558 & 0.8394 & 0.8644 & \textbf{0.9187} & 0.7264 & 0.8424 & 0.9152 & 0.8507 \\
% ImBD          & 0.9402 & 0.9890 & 0.9926 & 0.9882 & 0.9775 & 0.9101 & 0.9672 & 0.9657 & 0.9367 & 0.9450 \\
BiScope       & 0.8333 & 0.8733 & 0.9700 & 0.9600 & 0.9092 & 0.8533 & 0.8800 & 0.8800 & 0.9567 & 0.8925 \\
\cdashline{1-11}
% Fast-DetectGPT & 0.6796 & 0.9165 & 0.9162 & 0.9199 & 0.8580 & 0.8466 & 0.9270 & 0.9579 & 0.9579 & 0.8631 \\
Fast-DetectGPT & 0.9048 & 0.9588 & \textbf{0.9847} & 0.9800 & 0.9571 & 0.9019 & 0.9361 & \textbf{0.9768} & 0.9608 & 0.9439 \\
AdaDetectGPT  & 0.9072 & \textbf{0.9611} & 0.9832 & \textbf{0.9841} & \textbf{0.9589} & 0.9176 & \textbf{0.9400} & 0.9728 & \textbf{0.9610} & \textbf{0.9478} \\
\cdashline{1-11}
% Relative      & 71.0264 & 53.4451 & 80.0000 & 80.1332 & 71.0484 & 70.4942 & 60.8572 & 62.6561 & 7.3840 & 61.9021 \\
Relative (\parbox[c]{1em}{\tikz{\hspace*{3.9pt}\draw[blue, -latex] (0,0) -- (0,0.3);}}) & 2.4288 & 5.6095 & --- & 20.4444 & 4.2454 & 16.0326 & 6.0543 & --- & 0.3405 & 6.9929 \\
\hline
\end{tabular}
\end{adjustbox}
\end{table}

\begin{table}[t]
\centering
\caption{Detection of LLM-generated text under two adversarial attacks in black-box settings.}\label{tab:adversarial-attack}
\setlength{\tabcolsep}{3pt}
\begin{tabular}{lcccc|cccc}
\toprule
 & \multicolumn{4}{c}{\textbf{Paraphrasing}} & \multicolumn{4}{c}{\textbf{Decoherence}} \\
\cline{2-9}
DetectGPT & Xsum & Writing & PubMed & Avg. & Xsum & Writing & PubMed & Avg. \\
\midrule
Fast (GPT-J/GPT-2)   & 0.9178 & \textbf{0.9137} & 0.7944 & 0.8753 & 0.7884 & 0.9595 & 0.7870 & 0.8449 \\
Ada (GPT-J/GPT-2)    & \textbf{0.9225} & 0.9121 & \textbf{0.8029} & \textbf{0.8792} & \textbf{0.8765} & \textbf{0.9597} & \textbf{0.8284} & \textbf{0.8882} \\
\hline
Fast (GPT-J/Neo-2.7) & 0.9602 & \textbf{0.9185} & 0.7310 & 0.8699 & 0.8579 & 0.9701 & 0.7609 & 0.8630 \\
Ada (GPT-J/Neo-2.7)  & \textbf{0.9623} & 0.9181 & \textbf{0.7587} & \textbf{0.8797} & \textbf{0.9230} & \textbf{0.9704} & \textbf{0.8124} & \textbf{0.9019} \\
\hline
Fast (GPT-J/GPT-J)         & 0.9537 & \textbf{0.9458} & 0.7041 & 0.8679 & 0.8836 & \textbf{0.9869} & 0.7550 & 0.8752 \\
Ada (GPT-J/GPT-J)    & \textbf{0.9587} & 0.9449 & \textbf{0.7308} & \textbf{0.8781} & \textbf{0.9336} & 0.9864 & \textbf{0.8008} & \textbf{0.9070} \\
\bottomrule
\end{tabular}
\end{table}

\section{Conclusion}
We propose AdaDetectGPT, an adaptive LLM detector that learns a witness function $w$ to boost the performance of existing logits-based detectors. A natural approach to learning $w$ is to maximize the TNR of the resulting detector for a fixed FNR level $\alpha$. Our proposal has two novelties. First, by connecting Fast-DetectGPT’s statistic to martingale theory and applying the MCLT, we obtain a closed-form expression for the classification threshold that achieves FNR control at level $\alpha$. Second, the TNR is a highly complex function of $\alpha$ and $w$, which makes the learned witness function $\alpha$-dependent --- that is, it maximizes the TNR at a particular FNR level but does not guarantee optimality at other FNR levels. To address this, we derive a lower bound on the TNR and propose to learn $w$ by maximizing this lower bound. Our lower bound separates the effects of $\alpha$ and $w$: the witness function $w$ affects the lower bound only through $T_w^{(2)*}$, which is independent of $\alpha$. Consequently, the witness function that maximizes this lower bound simultaneously maximizes it across all FNR levels.

In our implementation, we opted to learn the witness function via a B-spline basis due to the straightforward nature of the optimization (finding the optimal witness function boils down to solving a system of linear equations) and its favorable theoretical properties \citep[the estimation error can attain Stone's optimal convergence rate,][see Appendix \ref{sec: smooth witness function} for more details]{stone1982optimal}. The number of basis functions can be selected in a data driven way via Lepski's method \citep{lepski1997optimal}.

\section*{Acknowledgments}
Chengchun Shi’s and Jin Zhu’s research was partially supported by the EPSRC grant EP/W014971/1. Hongyi Zhou’s and Ying Yang’s research was partially supported by NSFC 12271286 \& 11931001. Hongyi Zhou’s research was also partially supported by the China Scholarship Council. 

The authors thank the anonymous referees and the area chair for their insightful and constructive comments, which have led to a significantly improved version of the paper.

\bibliographystyle{icml2025}
\bibliography{reference}

%%%%%%%%%%%%%%%%%%%%%%%%%%%%%%%%%%%%%%%%%%%%%%%%%%%%%%%%%%%%

\newpage

% %%%%%%%%%%%%%%%%%%%%%%%%%%%%%%%%%%%%%%%%%%%%%%%%%%%%%%%%%%%%

\section*{NeurIPS Paper Checklist}

\begin{enumerate}

\item {\bf Claims}
    \item[] Question: Do the main claims made in the abstract and introduction accurately reflect the paper's contributions and scope?
    \item[] Answer: \answerYes{} % Replace by \answerYes{}, \answerNo{}, or \answerNA{}.
    \item[] Justification: We confirm the claims in abstract and introduction are accurately reflected by methodology, theory, and experiments in the paper.

\item {\bf Limitations}
    \item[] Question: Does the paper discuss the limitations of the work performed by the authors?
    \item[] Answer: \answerYes{} % Replace by \answerYes{}, \answerNo{}, or \answerNA{}.
    \item[] Justification: We discuss the limitation of the work in Appendix~\ref{sec:limitation}. 

\item {\bf Theory assumptions and proofs}
    \item[] Question: For each theoretical result, does the paper provide the full set of assumptions and a complete (and correct) proof?
    \item[] Answer: \answerYes{} % Replace by \answerYes{}, \answerNo{}, or \answerNA{}.
    \item[] Justification: The assumptions are intuitively illustrated in Section~\ref{sec:method} and rigorous stated in Appendix~\ref{sec:assumptions}. The proof is attached into Appendix~\ref{sec:proof}.

    \item {\bf Experimental result reproducibility}
    \item[] Question: Does the paper fully disclose all the information needed to reproduce the main experimental results of the paper to the extent that it affects the main claims and/or conclusions of the paper (regardless of whether the code and data are provided or not)?
    \item[] Answer: \answerYes{} % Replace by \answerYes{}, \answerNo{}, or \answerNA{}.
    \item[] Justification: The detailed setting for reproducibility are provided in Section~\ref{sec:experiment} and Appendix~\ref{sec:experiment-detail}.

\item {\bf Open access to data and code}
    \item[] Question: Does the paper provide open access to the data and code, with sufficient instructions to faithfully reproduce the main experimental results, as described in supplemental material?
    \item[] Answer: \answerYes{} % Replace by \answerYes{}, \answerNo{}, or \answerNA{}.
    \item[] Justification: The dataset and code in this paper are either publicly available or submitted as a new asset.

\item {\bf Experimental setting/details}
    \item[] Question: Does the paper specify all the training and test details (e.g., data splits, hyperparameters, how they were chosen, type of optimizer, etc.) necessary to understand the results?
    \item[] Answer: \answerYes{} % Replace by \answerYes{}, \answerNo{}, or \answerNA{}.
    \item[] Justification: The details for experiments are shown in Appendix~\ref{sec:experiment-detail}.

\item {\bf Experiment statistical significance}
    \item[] Question: Does the paper report error bars suitably and correctly defined or other appropriate information about the statistical significance of the experiments?
    \item[] Answer: \answerYes{} % Replace by \answerYes{}, \answerNo{}, or \answerNA{}.
    \item[] Justification: The results are accompanied by the error bars as can be seen from Figure~\ref{fig:separate}. 

\item {\bf Experiments compute resources}
    \item[] Question: For each experiment, does the paper provide sufficient information on the computer resources (type of compute workers, memory, time of execution) needed to reproduce the experiments?
    \item[] Answer: \answerYes{} % Replace by \answerYes{}, \answerNo{}, or \answerNA{}.
    \item[] Justification: We provided the compute resources in Appendix~\ref{sec:experiment-detail} -- \textbf{Hardware details}. 
    
\item {\bf Code of ethics}
    \item[] Question: Does the research conducted in the paper conform, in every respect, with the NeurIPS Code of Ethics \url{https://neurips.cc/public/EthicsGuidelines}?
    \item[] Answer: \answerYes{} % Replace by \answerYes{}, \answerNo{}, or \answerNA{}.
    \item[] Justification: The research conducted in the paper conform, in every respect, with the NeurIPS Code of Ethics.

\item {\bf Broader impacts}
    \item[] Question: Does the paper discuss both potential positive societal impacts and negative societal impacts of the work performed?
    \item[] Answer: \answerYes{} % Replace by \answerYes{}, \answerNo{}, or \answerNA{}.
    \item[] Justification: The broader impacts of this work is discussed in Appendix~\ref{sec:limitation}. 
    
\item {\bf Safeguards}
    \item[] Question: Does the paper describe safeguards that have been put in place for responsible release of data or models that have a high risk for misuse (e.g., pretrained language models, image generators, or scraped datasets)?
    \item[] Answer: \answerNA{} % Replace by \answerYes{}, \answerNo{}, or \answerNA{}.
    \item[] Justification: This research does not involve data or models that have a high risk for misuse.

\item {\bf Licenses for existing assets}
    \item[] Question: Are the creators or original owners of assets (e.g., code, data, models), used in the paper, properly credited and are the license and terms of use explicitly mentioned and properly respected?
    \item[] Answer: \answerYes{} % Replace by \answerYes{}, \answerNo{}, or \answerNA{}.
    \item[] Justification: In Section~\ref{sec:experiment} and Appendix~\ref{sec:experiment-detail}, we have explicitly cited or credited the assets used in the paper and explicitly mentioned the corresponding licenses.

\item {\bf New assets}
    \item[] Question: Are new assets introduced in the paper well documented and is the documentation provided alongside the assets?
    \item[] Answer: \answerYes{} % Replace by \answerYes{}, \answerNo{}, or \answerNA{}.
    \item[] Justification: The new asset is the implementation of the methods introduced in the paper. The documentation for the new asset is provided alongside.

\item {\bf Crowdsourcing and research with human subjects}
    \item[] Question: For crowdsourcing experiments and research with human subjects, does the paper include the full text of instructions given to participants and screenshots, if applicable, as well as details about compensation (if any)? 
    \item[] Answer: \answerNA{} % Replace by \answerYes{}, \answerNo{}, or \answerNA{}.
    \item[] Justification: This paper does not involve crowdsourcing nor research with human subjects.

\item {\bf Institutional review board (IRB) approvals or equivalent for research with human subjects}
    \item[] Question: Does the paper describe potential risks incurred by study participants, whether such risks were disclosed to the subjects, and whether Institutional Review Board (IRB) approvals (or an equivalent approval/review based on the requirements of your country or institution) were obtained?
    \item[] Answer: \answerNA{} % Replace by \answerYes{}, \answerNo{}, or \answerNA{}.
    \item[] Justification: This paper does not involve crowdsourcing nor research with human subjects.

\item {\bf Declaration of LLM usage}
    \item[] Question: Does the paper describe the usage of LLMs if it is an important, original, or non-standard component of the core methods in this research? Note that if the LLM is used only for writing, editing, or formatting purposes and does not impact the core methodology, scientific rigorousness, or originality of the research, declaration is not required.
    %this research? 
    \item[] Answer: \answerNA{} % Replace by \answerYes{}, \answerNo{}, or \answerNA{}.
    \item[] Justification: The LLM is used only for writing and editing, and it does not impact the core methodology. 

\end{enumerate}

\clearpage
\appendix

%%%%%%%%%%%%%%%%%%%%%%%%%%%%%%%%%%%%%%%%%%%%%%%%%%%%%%
% New Tables, Figures, Theorems in Supplement
%%%%%%%%%%%%%%%%%%%%%%%%%%%%%%%%%%%%%%%%%%%%%%%%%%%%%%
\renewcommand{\thetable}{S\arabic{table}}
\renewcommand{\thefigure}{S\arabic{figure}}
\renewcommand{\thelemma}{S\arabic{lemma}}
\renewcommand{\thetheorem}{S\arabic{theorem}}

\section{Additional related works: Watermarking-based detection}\label{sec:related-work-watermarking}
Watermarking embeds subtle signals into LLM-generated text to distinguish it from human-written text \citep[see][Section 4.2, for a recent review of LLM watermarking]{ji2025overview}.  \citet{aaronson2023watermarking} propose a watermarking technique based on Gumbel sampling. Follow-up works have focused on preserving text quality during watermarking \citep{christ2024undetectable,dathathri2024scalable,giboulot2024watermax,liu2024adaptive,wouters2024optimizewatermark,wu2024resilient}, enhancing watermark detection \citep{dathathri2024scalable,huo2024tokenspecific,cai2025statistical} and maintaining robustness against adversarial edits \citep{golowich2024edit}.

Our work is related to a line of research that frames watermark detection as a statistical hypothesis testing problem \citep[see e.g.,][]{kirchenbauer2023watermark,hu2024unbiased,kuditipudi2024robust,zhao2024provable,li2024robust,li2025statistical,chen2025a}. Under this framework, rejection of the null hypothesis (that no watermark is present) provides statistical evidence that the text was likely generated by an LLM.

\section{Details on the analytic example in Section~\ref{sec:method}} \label{section: discussion about witness function}

In this section, we provide rigorous discussion about the analytic example presented in Section~\ref{sec:method}. Noted that
\begin{eqnarray}
    \mathbb{E}_{\widetilde{X}_t\sim q, X_{<t}\sim p}\left\{w(\log q(X_t)  \right\} &=& q(1) w(\log q(1)) + q(0) w(\log q(0)),\nonumber\\
    \mathbb{E}_{X_{<t+1}\sim p} \left\{w(\log q(X_t)) \right\} &=& p_t(1) w( \log q(1)) + p_t(0)w(\log q(0)). \nonumber
\end{eqnarray}
It follows that 
\begin{align*}
    &\mathbb{E}_{\widetilde{X}_t\sim q, X_{<t}\sim p}\left\{w(\log q(X_t)  \right\} - \mathbb{E}_{X_{<t+1}\sim p} \left\{w(\log q(X_t)) \right\} 
    \\
    =& (q(1) - p_t(1)) \left[ w(\log{q(1)}) - w(\log{q(0)})\right]. \nonumber
\end{align*}
If $w$ is an identity function, i.e., $w(x) = x$, then the statistics \eqref{eqn:examplestatistics} becomes
\begin{eqnarray}
    \frac{1}{L} \log\Big(\frac{q(1)}{q(0)}\Big) \sum_{t=1}^L (q(1)-p_t(1)).\nonumber
\end{eqnarray}
In this case, \eqref{eqn:examplestatistics} converges to zero as $q\to 1/2$ regardless the distribution of $p_t$. However, if we consider adaptive witness function, the statistics in \eqref{eqn:examplestatistics} becomes 
\begin{eqnarray}
    \frac{1}{L} \left[w(\log q(1)) - w(\log q(0))\right] \sum_{t=1}^L (q(1)-p_t(1)). \nonumber
\end{eqnarray}
When $q(1)\neq 1/2$ (without generality, we assume $q(1) = 1-q(0) >1/2$), there always exists a witness function $w(z) = \mathbb{I}\left\{ z > \frac{\log q(1) + \log q(0)}{2} \right\}$ such that \eqref{eqn:examplestatistics} becomes
\begin{eqnarray}
    &&\frac{1}{L} \left[\mathbb{I}\{\log q(1) > \log q(0) \} - \mathbb{I}\{\log q(0) > \log q(1)\}\right]\sum_{t=1}^L (q(1) - p_t(1)) \nonumber\\
    &=&\frac{1}{L}\sum_{t=1}^L (q(1) - p_t(1)) = q(1)- \frac{1}{L}\sum_{t=1}^L p_t(1),\nonumber
\end{eqnarray}
which is independent of the log ratio.

\section{Experiment details}\label{sec:experiment-detail}

\textbf{Details for witness function estimation}.\label{sec:w-estimation}
In this part, we illustrate how we fetch external %human and machined-generate 
text datasets for training witness function in our experiments. 
As mentioned in the main text, when testing the performance of AdaDetectGPT on one dataset (e.g., XSum), we randomly select two other datasets (e.g., SQuAD and WritingPrompt) for training the witness function. This ensures the data for testing would not be included for training. %In the white-box setting, since $q^{(s)}$ is known, then the machine-generated text in the training datasets are generated by $q^{(s)}$. In the black-box setting, since $q^{(s)}$ is unknown, the machine-generated text in the training datasets are generated by the surrogate models. In the case where the sampling model and scoring model is different and both using surrogate models, then the generated texts comes from the surrogate scoring model. 

%\subsection{Details: the witness function estimation}

%In this section we elaborate the estimation of witness function. 
Recall that the population version of objective function for estimating $w$ is
\begin{align}\label{eqn:populationobj}
    \frac{\sum_t [\mathbb{E}_{X_{<t}\sim p, \widetilde{X}_t \sim q_t} w(\log q_t(\widetilde{X}_t|X_{<t}))]- \sum_t [\mathbb{E}_{X_{<t}\sim p, \widetilde{X}_t \sim p_t} w(\log q_t(\widetilde{X}_t|X_{<t}))] }{\sqrt{\sum_t \mathbb{E}_{X_{<t}\sim p} \textrm{Var}_{\widetilde{X}_t \sim q_t} (w(\log q_t(\widetilde{X}_t|X_{<t})))}}.
\end{align}
In the implementation, we made three modifications to this objective function to facilitate the computation, accommodate black-box settings with unavailable logits and handle prompts that are not explicitly included in the text: 
\begin{enumerate}[leftmargin=*]
    \item We replace the expectation $\mathbb{E}_{X_{<t}\sim p}$ in the numerator of \eqref{eqn:populationobj} with its empirical average over the human-authored passages $\{\bm{X}^{(i)}\}_i$. This leads to the following objective function:
\begin{align}\label{eqn:populationobj1}
    \frac{\sum_i \sum_t \mathbb{E}_{\widetilde{X}_t\sim q_t}[ w(\log q_t(\widetilde{X}_t|X^{(i)}_{<t}))] - \sum_i \sum_t [w(\log q_t(X^{(i)}_t|X^{(i)}_{<t}))] }{n\sqrt{ \sum_t \mathbb{E}_{X_{<t}\sim p}\textrm{Var}_{ \widetilde{X}_t \sim q_t} ( w(\log q_t(\widetilde{X}_t|X_{<t})))}}.
\end{align}
\item Taking the expectation $\mathbb{E}_{\widetilde{X}_t\sim q_t}$ and the variance $\textrm{Var}_{ \widetilde{X}_t \sim q_t}$ is time-consuming, as these operations need to be repeated $L$ times -- once at every token position $t$. To address this, we approximate $X_{<t}^{(i)}$ in the first term of the numerator of \eqref{eqn:populationobj1} by $\widetilde{X}_t^{(i)}$, sampled from the LLM distribution $q_t$. This, in turn, allows us to approximate the expectation by
\begin{align*}
    \frac{\sum_i \sum_t [w(\log q_t(\widetilde{X}_t^{(i)}|\widetilde{X}^{(i)}_{<t}))] - \sum_i \sum_t [w(\log q_t(X^{(i)}_t|X^{(i)}_{<t}))] }{n\sqrt{ \sum_t \mathbb{E}_{X_{<t}\sim p}\textrm{Var}_{ \widetilde{X}_t \sim q_t} ( w(\log q_t(\widetilde{X}_t|X_{<t})))}},
\end{align*}
where $\bm{\widetilde{X}}^{(i)}$ denotes the $i$th passage generated from $q$ by prompting the LLM to rewrite $\bm{X}^{(i)}$. As for the variance operator, we similarly approximate  $\mathbb{E}_{X_{<t}\sim p}$ in the denominator by  $\mathbb{E}_{X_{<t}\sim q}$. This allows us to upper bound the denominator by $n\sqrt{ \sum_t \textrm{Var}( w(\log q_t(\widetilde{X}_t|
\widetilde{X}_{<t})))}$, leading to the following lower bound of the objective function,
\begin{align*}
    \frac{\sum_i \sum_t [w(\log q_t(\widetilde{X}_t^{(i)}|\widetilde{X}^{(i)}_{<t}))] - \sum_i \sum_t [w(\log q_t(X^{(i)}_t|X^{(i)}_{<t}))] }{n\sqrt{ \sum_t \textrm{Var}( w(\log q_t(\widetilde{X}_t|
\widetilde{X}_{<t})))}},
\end{align*}
where the variance in the denominator can be estimated by the sampling variance estimator $$\widehat{\textrm{Var}}( w(\log q_t(\widetilde{X}_t|
\widetilde{X}_{<t})))=\frac{1}{n-1}\sum_{i=1}^n \Big[w(\log q_t(\widetilde{X}_t^{(i)}|
\widetilde{X}_{<t}^{(i)}))-\frac{1}{n}\sum_{j=1}^n w(\log q_t(\widetilde{X}_t^{(j)}|
\widetilde{X}_{<t}^{(j)}))\Big]^2.$$
\item To further simplify the objective function, we interchange the order of $\sum_t$ and $\widehat{\textrm{Var}}$ in the denominator, leading to 
\begin{align}\label{eqn:obj2}
    \frac{\sum_i \sum_t [w(\log q_t(\widetilde{X}_t^{(i)}|\widetilde{X}^{(i)}_{<t}))] - \sum_i \sum_t [w(\log q_t(X^{(i)}_t|X^{(i)}_{<t}))] }{n\sqrt{ \widehat{\textrm{Var}}(\sum_t w(\log q_t(\widetilde{X}_t|
\widetilde{X}_{<t})))}}.
\end{align}
Additionally, since the numerator incorporates both human- and LLM-authored text, we refine the denominator of Equation \eqref{eqn:obj2} by replacing the sampling variance estimator with a simple average of the estimators computed using human- and LLM-written text
\begin{eqnarray*}
    \frac{1}{2} \widehat{\textrm{Var}} \left(\sum_t w(\log q_t(\widetilde{X}_t| \widetilde{X}_{<t})) \right)+\frac{1}{2} \widehat{\textrm{Var}}\left(\sum_t w(\log q_t(X_t|X_{<t}))\right).
\end{eqnarray*}
This yields our final objective function, in the form of a two-sample $t$-test statistic,
\begin{eqnarray}\label{eqn:obj3}
    \frac{\sum_i \sum_t [w(\log q_t(\widetilde{X}_t^{(i)}|\widetilde{X}^{(i)}_{<t}))] - \sum_i \sum_t [w(\log q_t(X^{(i)}_t|X^{(i)}_{<t}))] }{n\sqrt{0.5\widehat{\textrm{Var}}(\sum_t w(\log q_t(\widetilde{X}_t|
\widetilde{X}_{<t})))+0.5\widehat{\textrm{Var}}(\sum_t w(\log q_t(X_t|X_{<t})))}}.
\end{eqnarray}
\end{enumerate}
Compared to its population-level version~\eqref{eqn:populationobj}, \eqref{eqn:obj3} is more suitable for black-box settings. In such settings, the distribution \(q\) is unknown, making the expectation and variance over \(q\) in~\eqref{eqn:populationobj} infeasible to compute. In contrast,~\eqref{eqn:obj3} relies on text generated by the LLM. Therefore, even without access to \(q\), we can still prompt the target LLM to produce \(\widetilde{\bm{X}}\). Likewise, for detecting text generated under specific prompts, we can incorporate these prompts into the rewriting process to produce \(\widetilde{\bm{X}}\).

We next discuss the computation of $\widehat{w}$ that maximizes \eqref{eqn:obj3}. To ease notation, we denoted $\log q(X^{(i)}_t|X^{(i)}_{<t})$ as $z^{(h)}_{it}$. Similarly, for the log-probabilities computed from machine-generated text, we define them as $z^{(m)}_{it}$s. Using these notations, \eqref{eqn:obj3} can be represented by
\begin{equation}\label{eq:implement-objective}
 \frac{1}{\sqrt{\textup{Var}(\sum_t w(z^{(m)}_t)) + \textup{Var}(\sum_t w(z^{(h)}_t))}} \left( \sum_{i=1}^{n} \frac{1}{L}\sum_{t=1}^{L} w(z^{(m)}_{it}) - \sum_{i=1}^{n} \frac{1}{L}\sum_{t=1}^{L} w(z^{(h)}_{it}) \right),
\end{equation}
up to some proportional constant. 

%Here, \eqref{eq:implement-objective} implicitly assumes that each passage has $L$ tokens. Actually, the computational procedure below can be easily extended to various sample sizes and various tokens number cases. 
Recall that we restrict the witness function to take a linear form of $w(z) = \phi(z)^\top \beta$, where $\phi(z)$ denotes the B-spline basis function \citep{de1978practical} and $\beta$ denotes the regression coefficients. Then numerator of~\eqref{eq:implement-objective} then becomes
\begin{align*}
    \sum_{i=1}^{n} \frac{1}{L} \sum_{t=1}^{L} \phi(z_{it}^{(m)})^\top \beta - \sum_{i=1}^{n} \frac{1}{L} \sum_{t=1}^{L} \phi(z_{it}^{(h)})^\top \beta,
    %= \left[ \sum_{i=1}^{n} \frac{1}{L} \sum_{t=1}^{L} \phi(z_{it}^{(h)})^\top - \sum_{i=1}^{n} \frac{1}{L} \sum_{t=1}^{L} \phi(z_{it}^{(m)})^\top \right]\beta
\end{align*}
whereas the denominator becomes $\sqrt{\beta^\top (\widehat{\Sigma}^{(m)}+\widehat{\Sigma
}^{(h)}) \beta}$, %and estimate $ \textup{Var}(\sum_t w(z^{(m)}_t))$ by $\beta^\top \widehat{\Sigma}^{(m)} \beta$, where in particular  
where $\widehat{\Sigma}^{(h)} = \sum_{i=1}^{n} \widehat{\Sigma}_i^{(h)}$, 
% \begin{align}
%     \frac{1}{L}(\mathbf{Z}^{(h)}_i)^{\top} \mathbf{Z}^{(h)}_i -  \widehat{\mu}^{(h)}_i (\widehat{\mu}^{(h)}_i)^{\top},
% \end{align}
% with 
\begin{align*}
& \widehat{\Sigma}_i^{(h)} = \frac{1}{L}(\mathbf{Z}^{(h)}_i)^{\top} \mathbf{Z}^{(h)}_i -  \widehat{\mu}^{(h)}_i (\widehat{\mu}^{(h)}_i)^{\top}, \\
& \mathbf{Z}_{i}^{(h)} = \left(\phi(z_{i1}^{(h)}), \ldots ,\phi(z_{iL}^{(h)})\right)^\top, \\
& \widehat{\mu}_i^{(h)} = \frac{1}{L} \sum_{t=1}^{L} \phi(z_{it}^{(h)})^\top, 
\end{align*}
and $\widehat{\Sigma}^{(m)}$ can be similarly defined. Consequently, the objective function can be rewritten as:
\begin{align*}
    &\frac{\left[ \sum_{i=1}^{n} \frac{1}{L} \sum_{t=1}^{L} \phi(z_{it}^{(m)})^\top - \sum_{i=1}^{n} \frac{1}{L} \sum_{t=1}^{L} \phi(z_{it}^{(h)})^\top \right]\beta}{\sqrt{\beta^\top (\widehat{\Sigma}^{(h)} + \widehat{\Sigma}^{(m)})\beta}}
    \\
    =& \left[ \sum_{i=1}^{n} \frac{1}{L} \sum_{t=1}^{L} \phi(z_{it}^{(m)})^\top - \sum_{i=1}^{n} \frac{1}{L} \sum_{t=1}^{L} \phi(z_{it}^{(h)})^\top \right] \beta
    \times \frac{1}{\| (\widehat{\Sigma}^{(h)} + \widehat{\Sigma}^{(m)})^{1/2} \beta\|_2}
    \\
    =&  \left[ \sum_{i=1}^{n} \frac{1}{L} \sum_{t=1}^{L} \phi(z_{it}^{(m)})^\top - \sum_{i=1}^{n} \frac{1}{L} \sum_{t=1}^{L} \phi(z_{it}^{(h)})^\top \right] \times (\widehat{\Sigma}^{(h)} + \widehat{\Sigma}^{(m)})^{-\frac{1}{2}} \alpha,
\end{align*}
where $\alpha = (\widehat{\Sigma}^{(h)} + \widehat{\Sigma}^{(m)})^{1/2} \beta / \Vert(\widehat{\Sigma}^{(h)} + \widehat{\Sigma}^{(m)})^{1/2} \beta\Vert_2$ whose $\ell_2$ norm equals 1. It is immediate to see that the argmax $\widehat{\alpha}$ has a closed-form expression, 
\begin{align*}
    \widehat\alpha = \frac{\widetilde\alpha}{\|\widetilde\alpha\|_2}
\end{align*}
where
\begin{align}
    \widetilde\alpha = (\widehat{\Sigma}^{(h)} + \widehat{\Sigma}^{(m)})^{-\frac{1}{2}} \left[ \sum_{i=1}^{n} \frac{1}{L} \sum_{t=1}^{L} \phi(z_{it}^{(m)}) - \sum_{i=1}^{n} \frac{1}{L} \sum_{t=1}^{L} \phi(z_{it}^{(h)}) \right].
\end{align}
This leads to
\begin{align*}
\widehat\beta =  (\widehat{\Sigma}^{(h)} + \widehat{\Sigma}^{(m)})^{-1/2}\widehat{\alpha}.
\end{align*}

and $\widehat{w}(z) = \phi(z)^\top \widehat{\beta}$. 

%From Section~\ref{sec:tuning}, considering the increase of \texttt{n\_base} would raise higher computational cost, and thus, we recommend to set \texttt{n\_base}=16 and \texttt{order}=2. 

% The Black-box Setting. In real-world situations, there could be instances where we lack knowledge about the specific source models employed for content generation. This necessitates the development of a versatile detector capable of identifying texts generated by a variety of automated systems. We term this scenario the black-box setting, where the objective is to differentiate between machine generated texts produced by diverse, unidentified models and those composed by humans. In this context, the term ``black box'' signifies that we lack access to information about the source model or any details pertaining to it.  

\textbf{Pre-trained language models.} We assess the performance of our method using text generated from various pre-trained language models outlined in Table~\ref{tab:llms}. Following the setting in \citet{bao2024fastdetectgpt}, for the models with over 6B parameters, we employ half-precision (\texttt{torch.float16}), otherwise, we use full-precision (\texttt{torch.float32}).
\begin{table}[H]
    \centering
    \caption{Description of the source models that is used to produce machine-generated text. $^\dagger$: we present the address of models in \href{Huggingface}{https://huggingface.co/}.}
    \begin{tabular}{l|cc}
    \toprule
    Name & Model$^\dagger$ & Scale (Billion)  \\
    \midrule
    GPT-2 \citep{radford2019language} & openai-community/gpt2-xl & 1.5B \\
    GPT-Neo \citep{gpt-neo} & EleutherAI/gpt-neo-2.7B & 2.7B \\
    OPT-2.7 \citep{zhang2022opt} & facebook/opt-2.7b & 2.7B \\
    GPT-J \citep{gpt-j}  & EleutherAI/gpt-j-6B & 6B \\
    Qwen2.5 \citep{yang2024technical}   & Qwen/Qwen2.5-7B & 7B \\
    Mistral \citep{albert2024mistral}   & mistralai/Mistral-7B-v0.3 & 7B \\
    Llama \citep{llama3modelcard}   & meta-llama/Meta-Llama-3-8B & 8B \\
    GPT-NeoX \citep{GPT-NeoX}   & EleutherAI/gpt-neox-20b & 20B \\
    \bottomrule
    \end{tabular}
    \label{tab:llms}
\end{table}

\lstdefinelanguage{json}{
    numbers=none,
    numberstyle=\small,
    frame=single,
    rulecolor=\color{black},
    showspaces=false,
    showtabs=false,
    breaklines=true,
    postbreak=\raisebox{0ex}[0ex][0ex]{\ensuremath{\color{gray}\hookrightarrow\space}},
    breakatwhitespace=true,
    basicstyle=\ttfamily\small,
    upquote=true,
    morestring=[b]",
    stringstyle=\color{string},
    literate=
     *{0}{{{\color{numb}0}}}{1}
      {1}{{{\color{numb}1}}}{1}
      {2}{{{\color{numb}2}}}{1}
      {3}{{{\color{numb}3}}}{1}
      {4}{{{\color{numb}4}}}{1}
      {5}{{{\color{numb}5}}}{1}
      {6}{{{\color{numb}6}}}{1}
      {7}{{{\color{numb}7}}}{1}
      {8}{{{\color{numb}8}}}{1}
      {9}{{{\color{numb}9}}}{1}
      {\{}{{{\color{delim}{\{}}}}{1}
      {\}}{{{\color{delim}{\}}}}}{1}
      {[}{{{\color{delim}{[}}}}{1}
      {]}{{{\color{delim}{]}}}}{1},
}

\textbf{Setup of the closed-source LLMs}. For the gpt-4o, the version is set as \texttt{gpt-4o-2024-08-06}. The generation process by sending the following messages to the service. Message for XSum and Writing is the same as that described in Section C.2 in \citet{bao2024fastdetectgpt}. We describe that for Yelp and Essay are:
\begin{lstlisting}[language=json]
[
  {'role': 'system', 'content': 'You are a Review writer on Yelp.'},
  {'role': 'user', 'content': 'Please write an article with about 150 words starting exactly with: <prefix>'},
]
\end{lstlisting}
and 
\begin{lstlisting}[language=json]
[
  {'role': 'system', 'content': 'You are a student of high school and university level. And now, you are an Essay writer.'},
  {'role': 'user', 'content': 'Please write an essay with about 200 words starting exactly with: <prefix>'},
]
\end{lstlisting}
respectively.

For Claude-3.5-Haiku, the system instruction was set analogously (e.g., Yelp review or essay writer), while the user role contained the corresponding content prompt. 

For Gemini, the instruction was fed into the \texttt{system\_instruction} parameter, with a value identical to the concatenation of the system content and the user content used for GPT-4o. 

For all closed-source models, the temperature parameter is set to 0.8 to encourage the generated text to be creatively diverse and less predictable.

\textbf{Setting on experiments with adversarial attacks}. In Table~\ref{tab:adversarial-attack}, we have conducted experiments to evaluate the robustness of AdaDetectGPT against 2 adversarial attacks: (i) paraphrasing, where an LLM is instructed to rephrase human-written text, and (ii) decoherence, where the coherence LLM-generated text is intentionally reduced to avoid detection. These experiments were carried out across 3 datasets and 3 types of sampling and scoring models setup, resulting in a total of 18 settings. Both adversarial attacks were implemented following \citet{bao2024fastdetectgpt}.

% Results are reported in the Table~\ref{tab:adversarial-attack}. It can be seen that, AdaDetectGPT consistently achieves higher AUCs than Fast-DetectGPT in most of the 18 scenarios. And the improvement reaches up to 10\% for paraphrasing and up to 85\% for decoherence. These results suggest that AdaDetectGPT attains more robust performance against adversarial attacks.

\textbf{Implementations of baselines}. For the baselines considered in our experiments, we use the existing implementation provided in \url{https://github.com/baoguangsheng/fast-detect-gpt}, which is distributed in the MIT License. We run DetectGPT and NPR with default 100 perturbations with the T5 model \citep{2020t5} and run DNA-GPT with a truncate-ratio of 0.5 and 10 prefix completions per passage. 

\textbf{Evaluation Metric}. We measure the detection accuracy by AUC (short for ``area under the curve''). AUC ranges from 0.0 to 1.0, an AUC of 1.0 indicates a perfect classifier and vice versa. The relative improvement of AdaDetectGPT over FastDetectGPT is calculated by $\frac{\textup{AdaDetectGPT} - \textup{FastDetectGPT}}{1.0 - \textup{FastDetectGPT}}$, which represents how much improvement has been made relative to the maximum possible improvement for FastDetectGPT.

\textbf{Hardware details}. Most of experiments are run on a Tesla A100 GPU (40GB) with 10 vCPU Intel Xeon Processor and 72GB RAM. For the experiments where the source model is GPT-NeoX, we run on a 
H20-NVLink (96GB) GPU with 20 vCPU Intel(R) Xeon(R) Platinum and 200GB RAM. 

% Methods marked with ``{\it (fixed)}'' use the fixed models to detect texts from different sources (the black-box setting), where DetectGPT uses T5-3B/GPT-Neo as the perturbation/scoring models and Fast-DetectGPT uses GPT-J/GPT-Neo as the sampling/scoring models. 

% \begin{figure}
%     \centering
%     \includegraphics[width=1.0\linewidth]{figure/prompts.png}
%     \caption{Caption}
%     \label{fig:prompt}
% \end{figure}

\section{Technical assumptions}\label{sec:assumptions}

In this section, we list the assumptions required for the theorems presented in Section~\ref{sec:method} to hold, and discuss when they are expected to hold and how they may be relaxed.

%\subsection{Assumptions}\label{sec:assumptions}

%We work under the following assumptions on the data generating process.

\begin{assumption}[Margin]
With $T_{w} (\bullet)$ defined as in \eqref{eqn:Tw} and $w^* (\bullet)$ defined as the optimizer of \eqref{equation: T_w^*}, for any $\alpha \in (0,1)$ there are constants $\delta_\alpha, C_\alpha$ depending only on $\alpha$ such that for any $x \leq \delta_\alpha$ it holds that $\mathbb{P}_{\boldsymbol{X} \sim p} (|T_{w^*} (\boldsymbol{X}) - z_\alpha| \leq x ) \leq C_\alpha x$. 
\label{assumption: margin condition}
\end{assumption}
We also require the following technical conditions hold in order to obtain TNR lower bound and FNR control (Theorem \ref{thm:TNRlowerbound} and Theorem \ref{thm:FNR}). 

\begin{assumption}[Minimum eigenvalue]
For each $t = 1, \dots, L$ introduce the quantities
\begin{align*}
& \mu_t^{(1)} = \mathbb{E}_{X_{<t} \sim p} \mathbb{E}_{\widetilde{X}_t \sim q_t} \phi ( \log q_t ( \widetilde{X}_t \mid X_{<t} ) ), \\
& \Sigma_t = \mathbb{E}_{X_{<t} \sim p} \mathbb{E}_{\widetilde{X}_t \sim q_t}  \phi ( \log q_t ( \widetilde{X}_t \mid X_{<t} ) ) \phi ( \log q_t ( \widetilde{X}_t \mid X_{<t} ) )^\top - \mu_t^{(1)} ( \mu^{(1)}_t )^\top. 
\end{align*}
There are absolute constant $C > 0$ and $\gamma > 0$ such that $\lambda_{\text{min}} (\Sigma_t) \geq C d^{-\gamma}$ for all $t$. 
\label{assumption: eigenvalue condition}
\end{assumption}

\begin{comment}
\begin{assumption}[Stochastic dominance]\label{ass:stochastic-dominance}
    For any witness function $w\in \mathcal{W}$, 
    \begin{equation*}
        \sum_t \mathbb{E}_{\widetilde{X}_t\sim q_t}w(\log q_t(\widetilde{X}_t|X_{<t})) - \mathbb{E}_{\widetilde{X}_t\sim p_t}w(\log q_t(\widetilde{X}_t|X_{<t}))\geq 0
    \end{equation*}
    holds almost surely.
\end{assumption}
\end{comment}

\begin{assumption}[Equal variance]\label{ass:equal-variance}
    For any non-constant witness function $w$, define 
    \begin{align*}
        \sigma_{q,L}^2 \coloneqq \frac{1}{L} \sum_{t=1}^L \textup{Var}_{\widetilde{X}_t\sim q_t}\left( w(\log q_t(\widetilde{X}_t|\widetilde{X}_{<t})) \right),\\
        \sigma_{p,L}^2 \coloneqq \frac{1}{L} \sum_{t=1}^L \textup{Var}_{\widetilde{X}_t\sim p_t}\left( w(\log q_t(\widetilde{X}_t|\widetilde{X}_{<t})) \right).
    \end{align*}
    $\sigma_{q,L}^2, \sigma_{p,L}^2 $ are lower bounded by some constant $\sigma_w^2 >0$ almost surely. Moreover, $\sigma_{q,L} - \sigma_{p,L} \to 0 $ in probability as $L \to \infty$.
\end{assumption}

\begin{assumption}\label{ass:variance-ratio}
    For any witness function $w$, define
    \begin{align*}
        \bar\sigma_{q,L}^2 = \frac{1}{L} \sum_{t=1}^L \textup{Var}_{\bm{X}\sim q}\left( w(\log q_t(\widetilde{X}_t|\widetilde{X}_{<t})) \right) ,\\
        \bar\sigma_{p,L}^2 = \frac{1}{L} \sum_{t=1}^L \textup{Var}_{\bm{X}\sim p}\left( w(\log q_t(\widetilde{X}_t|\widetilde{X}_{<t})) \right) .
    \end{align*}
    If $\bm{X}\sim q$, then $\bar\sigma_{q,L}^2/\sigma_{q,L}^2 \to 1$ in probability. If $\bm{X}\sim p$, then $\bar\sigma_{p,L}^2/\sigma_{p,L}^2 \to 1$ in probability.
\end{assumption}
Conditions similar to Assumption \ref{assumption: margin condition} are commonly assumed; see, for instance, \cite{audibert2007fast, qian2011performance, luedtke2016statistical, shi2020breaking, shi2020sparse, shi2022statistical}. Assumption \ref{assumption: eigenvalue condition} is mild, since the constant $\gamma$ is allowed to be arbitrarily large. 
%As for Assumption \ref{ass:stochastic-dominance}, notice that the witness function is trained to maximize the mean discrepancy.  Therefore, it is reasonable to assume that the discrepancy is non-negative. 
Assumption \ref{ass:equal-variance} basically requires the conditional variance of logits be asymptotically equivalent for human-authored and machine-generated passages. This assumption is not overly restrictive, as the variance discrepancy between the two types of passages is relatively small in our dataset (see Table \ref{tab:assumption-practice}, where the ratio f the variances are very closed to 1). Assumption \ref{ass:variance-ratio} is commonly assumed in martingale central limit theorem literature, see e.g. \cite{Bolthausen_1982_exact, hall2014martingale}. Our empirical results further support the validity of Assumption~\ref{ass:variance-ratio}: the sample mean of the ratio remains nearly constant as a function of $L$, while its variance (in parentheses) approaches zero as $L \to \infty$ across all three datasets (see Table~\ref{tab:assumption-practice}), suggesting that this condition is practical. Furthermore, two commonly employed hypothesis tests --- the Kolmogorov–Smirnov (KS) test and the Shapiro–Wilk (SW) test --- are conducted to evaluate whether the proposed statistic follows a normal distribution. As shown in Table~\ref{tab:normality-test}, almost all $p$-values exceed 0.1 (most by a large margin), indicating that our test statistic passes normality tests in most cases. These test results also provide strong empirical support for the validity of MCLT regularity conditions.

\begin{table}[H]
\centering
\caption{Sample mean and variance (in parentheses) of the ratio evaluated on 3 datasets as $L$ increases.}\label{tab:assumption-practice}
\begin{tabular}{lcccccc}
\toprule
$L$ & 100 & 150 & 200 & 250 & 300 & 350 \\
\midrule
XSum    & 1.09(0.12) & 1.07(0.09) & 1.06(0.08) & 1.04(0.07) & 1.01(0.06) & 0.99(0.05) \\
SQuAD   & 1.03(0.10) & 1.02(0.07) & 1.02(0.06) & 1.02(0.05) & 1.01(0.05) & 1.03(0.04) \\
Writing & 1.10(0.06) & 1.09(0.05) & 1.09(0.04) & 1.08(0.03) & 1.07(0.03) & 1.05(0.03) \\
\bottomrule
\end{tabular}
\end{table}

\begin{table}[H]
\centering
\caption{$p$-values of KS and SW tests across 3 datasets and 2 source LLMs.}\label{tab:normality-test}
\begin{tabular}{llccc}
\toprule
LLM & Test & XSum & SQuAD & Writing \\
\midrule
GPT-Neo & KS & 0.72 & 0.54 & 0.18 \\
GPT-Neo & SW & 0.50 & 0.65 & 0.89 \\
GPT-2   & KS & 0.10 & 0.52 & 0.28 \\
GPT-2   & SW & 0.37 & 0.026 & 0.14 \\
\bottomrule
\end{tabular}
\end{table}

%b is made to simplify the theoretical exposition. In principle one could allow $d$ to diverge with $L$ and $n$ and the rate obtained in Theorem \ref{thm: finite sample AUCs} would change accordingly. Note however that for the problem being studied the choice of $d$ cannot be guided by classical spline theory. Typically one chooses $d$ according to the smoothness of the target function in order to attain the best rate, however as shown in \ref{section: discussion about witness function}, in some cases the ``best'' $w (\cdot)$ may be an indicator function, which is non-smooth. 

%For conditions (D1), if the denominator is replaced by $\left(\textup{Var}_{\widetilde{X}\sim p_H} \sum_{j=1}^L\left(w(\log p_\theta(\widetilde{X}|x_{<j})\right)\right)^{-1/2}$, the condition naturally holds follows from martingale central limit theorem. Therefore, we actually requires the ratio between the denominator and the deviation of numerator converges to some distribution independent of $w$. This condition holds if the temporal dependence can be asymptotically neglected in which case the ratio converges to constant $1$. Condition (D2) is similar as (D1). 
 
\section{Proofs}\label{sec:proof}

\subsection{Notations}

Throughout the proofs we will make use of the following notation. We will denote absolute constants by $\kappa_1, \kappa_2, \cdots$. For a sequence of random variables $\{ X_n \mid n \geq 1 \}$ with distribution functions $\{ F_{X_n} | n \geq 1 \}$ and some (possibly degenerate) random variable $Y$ with distribution function $F_Y$ we write $X_n \stackrel{p}{\rightarrow} Y$ as $n \rightarrow \infty$ if $\lim\limits_{n \rightarrow \infty} \mathbb{P} ( |X_n - Y| > \delta ) = 0 $ for all $\delta > 0$, and $X_n \stackrel{d}{\rightarrow} Z$ if $\lim\limits_{n \rightarrow \infty } F_{X_n} (x) = F_Y(x)$ at every continuity point of $F_Y(\cdot)$. For a vector $x = (x_1, \dots, x_d )^\top \in \mathbb{R}^d$ we write $\| x \|_p = (\sum_{j=1}^dx_j^p)^{1/p}$ with $0 < p < \infty$ for its $\ell_p$-norm.

\subsection{Preparatory results}
We introduce three auxiliary results in this section. Theorem \ref{theorem: bounded difference} presents a concentration inequality that is critical to establishing the learning guarantees of AdaDetectGPT in Theorem \ref{thm:TNR}. Theorem \ref{thm:MCLT} formally states the MCLT. Finally, Lemma \ref{thm: convergence-rate-MCLT} can be viewed as a non-asymptotic version of Theorem \ref{thm:MCLT} which provides an explicit error bound on the accuracy of the normal approximation in the MCLT.
\begin{theorem}[Bounded differences inequality]\label{theorem: bounded difference}
Let $\mathcal{X}$ be a measurable space. A function $f: \mathcal{X}^n \to \mathbb{R}$ has the bounded difference property for some constants $c_1, \dots, c_n$, i.e., for each $i = 1, \dots, n$,
\begin{equation}
\sup_{\substack{x_1, \dots, x_n \\ x_i' \in \mathcal{X}}} \left | f \left ( x_1, \dots, x_{i-1}, x_i, x_{i+1}, x_n \right ) - f \left ( x_1, \dots, x_{i-1}, x_i', x_{i+1}, \dots x_n \right ) \right | \leq c_i. 
\label{equation: bounded difference property}
\end{equation}
Then, if $X_1, \dots, X_n$ is a sequence of identically distributed random variables and (\ref{equation: bounded difference property}) holds, putting $Z = f \left ( X_1, \dots, X_n \right )$ and $\nu = \frac{1}{4} \sum_{i=1}^n c_i^2$ for any $t > 0$, it holds that
\begin{equation*}
\mathbb{P} \left ( Z - \mathbb{E} \left ( Z \right ) > t \right ) \leq e^{-t^2 / \left ( 2 \nu \right )}.  
\end{equation*}
\end{theorem}

\begin{proof}[Proof of Theorem \ref{theorem: bounded difference}]
    See Section 2 in \cite{wainwright2019high}.
\end{proof}

\begin{theorem}[Martingale central limit theorem]\label{thm:MCLT}
Let $\left \{ M_{n,i} \mid 1 \leq i \leq k_n, n \geq 1 \right \}$ be a zero mean square integrable martingale array with respect to the filtrations $\left \{ \mathcal{F}_{n,i} \mid 1 \leq i \leq k_n, n \geq 1 \right \}$ having increments $X_{n,i} = M_{n,i} - M_{n,i-1}$. If the following conditions hold
\begin{itemize}[leftmargin=*]
    \item[] \textbf{C1:} $\sum_{i=1}^{k_n} \mathbb{E} \left [ X_{n,i} \mathbf{1}_{\left \{ \left | X_{n,i} \right | > \delta \right \}} \mid \mathcal{F}_{n,i-1} \right ] \stackrel{p}{\rightarrow} 0$ as $n \rightarrow \infty$ for all $\delta > 0$
    \item[] \textbf{C2:} $\sum_{i=1}^{k_n} \mathbb{E} \left [ X_{n,i}^2 \mid \mathcal{F}_{n,i-1} \right ] \stackrel{p}{\rightarrow} \sigma^2$ as $n \rightarrow \infty$
    \item[] \textbf{C3:} the $\sigma$-fields are nested: $\mathcal{F}_{n,i} \subseteq \mathcal{F}_{n+1,i}$ for $1 \leq i \leq k_n$ and $n \geq 1$
\end{itemize}
then $M_{n,k_n} \stackrel{d}{\rightarrow} Z$ as $n \rightarrow \infty$, where $Z \sim \mathcal{N} \left ( 0, \sigma^2 \right )$. 
\label{theorem: martingale CLT}
\end{theorem}

\begin{proof}
See Corollary 3.1 in \cite{hall2014martingale} and Theorem 2 in \cite{Bolthausen_1982_exact}.
\end{proof}

\begin{lemma}[Convergence rates in MCLT]\label{thm: convergence-rate-MCLT}
    Let $\bm{X} = (X_1,\ldots X_n)$ be sequences of real valued random variables satisfying for all $1\leq t\leq n$, 
    \begin{equation*}
        \mathbb{E} (X_t | X_{<t}) = 0\quad \textup{almost surely.}
    \end{equation*}
    Let $\sigma_t^2 = \mathbb{E} (X_t^2\big| X_{<t})$, $\bar\sigma_t^2 = \mathbb{E}(X_t^2)$, $s_n^2 = \sum_{t=1}^n \bar\sigma_t^2$ and $V_n^2 = \sum_{t=1}^n \sigma_t^2/s_n^2$. Suppose $|X_n|$ is bounded by some constant almost surely for all $n$ and $s_n/\sqrt{n}$ is bounded away from zero. Then % and $V_n^2 \to 1$ as $n\to \infty$, then
    \begin{equation*}
        \sup_{z\in \mathbb{R}} \left| \mathbb{P}\left( \frac{\sum_{t=1}^n X_t}{\sqrt{\sum_{t=1}^n \sigma_t^2}} \leq z \right) - \Phi(z) \right| = O\left(\frac{\log n}{\sqrt{n}} +( \mathbb{E}|V_n^2 -1|)^{1/3}\right),
    \end{equation*}
    where $\Phi(\bullet)$ is the cumulative distribution function of standard normal distribution.
\end{lemma}
\begin{proof}
    It follows from Corollary 1 of \cite{Bolthausen_1982_exact} and the condition $s_n/\sqrt{n}$ is bounded away from zero that 
    \begin{equation*}
        \sup_{z\in \mathbb{R}} \left| \mathbb{P}\left( \frac{\sum_{t=1}^n X_t}{s_n} \leq z \right) - \Phi(z) \right| = O\left(\frac{n\log n}{s_n^3} +( \mathbb{E}|V_n^2 -1|)^{1/3}\right)=O\left(\frac{\log n}{\sqrt{n}} +( \mathbb{E}|V_n^2 -1|)^{1/3}\right).
    \end{equation*}
    It follows that
    \begin{eqnarray}\label{eqn:berryessen-eq1}
        && \sup_{z\in \mathbb{R}} \left| \mathbb{P}\left( \frac{\sum_{t=1}^n X_t}{\sqrt{\sum_{t=1}^n \sigma_t^2}} \leq z \right) - \Phi(z) \right| \nonumber\\
        &=& \sup_{z\in \mathbb{R}} \left| \mathbb{P}\left( \frac{\sum_{t=1}^n X_t}{s_n} \leq z + \left(\frac{\sum_{t=1}^n X_t}{\sqrt{\sum_{t=1}^n \sigma_t^2}}\right)(V_n -1)  \right) - \Phi(z) \right| \nonumber\\
        &\leq& \sup_{z\in \mathbb{R}}\mathbb{E} \left| \mathbb{P}\left( \frac{\sum_{t=1}^n X_t}{s_n} \leq z + \left(\frac{\sum_{t=1}^n X_t}{\sqrt{\sum_{t=1}^n \sigma_t^2}}\right)(V_n -1)  \right) - \Phi\left(z + \left(\frac{\sum_{t=1}^n X_t}{\sqrt{\sum_{t=1}^n \sigma_t^2}}\right)(V_n -1)\right) \right| \nonumber\\
        &&+ \sup_{z\in\mathbb{R}}\mathbb{E}\left| \Phi\left(z + \left(\frac{\sum_{t=1}^n X_t}{\sqrt{\sum_{t=1}^n \sigma_t^2}}\right)(V_n -1)\right) - \Phi(z)\right|\nonumber\\
        &\leq& O\left(\frac{\log n}{\sqrt{n}} +( \mathbb{E}|V_n^2 -1|)^{1/3}\right) + \sup_{z\in \mathbb{R}}|\Phi'(z)| \times \mathbb{E}\left| \left(\frac{\sum_{t=1}^n X_t}{\sqrt{\sum_{t=1}^n \sigma_t^2}}\right)(V_n -1) \right|.
    \end{eqnarray}
    By the definition $\Phi(z)$, we have that 
    \begin{equation}\label{eqn:berryessen-eq2}
        \sup_{z\in \mathbb{R}}|\Phi'(z)| = \sup_{z\in\mathbb{R}} \frac{1}{\sqrt{2\pi}}\exp(-z^2/2) \leq \frac{1}{\sqrt{2\pi}}.
    \end{equation}
    Leveraging the facts that $\sup_t |X_t|$ are upper bounded and $s_n/\sqrt{n}$ is lower bounded, %$s_n^2/n \to s^2 >0$ and $V_n \to 1$ in probability, 
    we obtain that 
    \begin{eqnarray}\label{eqn:berryessen-eq3}
        \mathbb{E}\left| \left(\frac{\sum_{t=1}^n X_t}{\sqrt{\sum_{t=1}^n \sigma_t^2}}\right)(V_n -1) \right| &=& \mathbb{E}\left| \left(\frac{\sum_{t=1}^n X_t}{s_n}\right)V_n(V_n -1) \right|\nonumber \\
        &\leq&\left(\mathbb{E}\left\{\frac{\left(\sum_{t=1}^n X_t\right)^{3/2}}{s_n^{3/2}}V_n^{3/2}\right\}\right)^{2/3}\left(\mathbb{E}\left\{(V_n-1)^3\right\}\right)^{1/3} \nonumber\\
        &=&O\left(\left\{\mathbb{E}|V_n-1|^2\right\}^{1/3}\right)\nonumber\\
        &=& O\left(\left\{\mathbb{E}|V_n^2-1|\right\}^{1/3}\right), %\leq \left\Vert \frac{\sum_{t=1}^n X_t}{s_n}V_n^2 \right\Vert_\infty \mathbb{E}|V_n^2-1|
    \end{eqnarray}
    where the third-to-last line is derived from H$\ddot{\text{o}}$lder inequality, and the second-to-last equality holds due to the boundedness of $V_n$. Combining equations \eqref{eqn:berryessen-eq1}, \eqref{eqn:berryessen-eq2} and \eqref{eqn:berryessen-eq3}, we obtain
    \begin{eqnarray}
         \sup_{z\in \mathbb{R}} \left| \mathbb{P}\left( \frac{\sum_{t=1}^n X_t}{\sqrt{\sum_{t=1}^n \sigma_t^2}} \leq z \right) - \Phi(z) \right| \leq O\left(\frac{\log n}{\sqrt{n}} +( \mathbb{E}|V_n^2 -1|)^{1/3}\right),
    \end{eqnarray}
    which finishes the proof.
\end{proof}

\begin{lemma}\label{lem:Taylorexpansion-of-normal-cdf}
    Let $\Phi(\bullet)$ be the cumulative distribution function of standard normal distribution and $\Phi'(\bullet)$ be its derivative. Then for any random variable $X$, 
    \begin{eqnarray}
        \mathbb{E}\Phi(z_\alpha +X) \geq \min\{ 1-\alpha, \alpha + \Phi'(z_\alpha)\mathbb{E}X\}, \nonumber
    \end{eqnarray}
    where $0<\alpha<1/2$, $z_\alpha$ is the $\alpha$-th quantile of standard normal distribution.
\end{lemma}

\begin{proof}[Proof of Lemma \ref{lem:Taylorexpansion-of-normal-cdf}]
    Since $\Phi'(x) = (\sqrt{2\pi})^{-1} \exp\left(-x^2/2\right)$, we noted that $\Phi'(x)\geq \Phi'(z_\alpha)$ holds if and only if $z_\alpha \leq x \leq z_{1-\alpha}$.
    Therefore, if $0\leq X < z_{1-\alpha} - z_\alpha$, then by the mean value theorem,
    \begin{eqnarray}
        \Phi(z_\alpha +X) = \Phi(z_\alpha) + \Phi'(\xi)X \geq \alpha + \Phi'(z_\alpha)X, \nonumber
    \end{eqnarray}
    where $\xi$ lies between $z_\alpha$ and $z_{1-\alpha}$. If $X \leq 0$, then
    \begin{eqnarray}
        \Phi(z_\alpha +X) = \Phi(z_\alpha) + \Phi'(\eta)X \geq \alpha + \Phi'(z_\alpha)X, \nonumber
    \end{eqnarray}
    where $\eta$ lies between $X$ and $z_{\alpha}$. Moreover, if $X \geq  z_{1-\alpha} - z_\alpha$, then $z_\alpha +X \geq z_{1-\alpha}$, It follows that $\Phi(z_\alpha +X) \geq \Phi(z_{1-\alpha}) =1-\alpha$. Therefore,
    \begin{eqnarray}
        \mathbb{E}\Phi(z_\alpha +X)  &\geq&  \mathbb{E} \min\left\{ \alpha + \Phi'(z_\alpha)X,1-\alpha  \right\} \nonumber\\
        &\geq& \min\left\{\alpha + \Phi'(z_\alpha)\mathbb{E}X, 1-\alpha\right\},\nonumber
    \end{eqnarray}
    where the last inequality follows from Jensen's inequality. This finishes the proof.
\end{proof}

In Lemma~\ref{lem:bound for beta} below, we provide an upper bound for the parameter estimation error. Before doing so, we define
\begin{align}
& Q_* \left ( \beta \right ) = \left \{ \beta^\top \Sigma \beta \right \}^{-\frac{1}{2}} \beta^\top \mu \label{equation: limit functional} \\
& \widehat{Q}_n \left ( \beta \right ) = \left \{ \beta^\top \widehat{\Sigma}_n \beta \right \}^{-\frac{1}{2}} \beta^\top \widehat{\mu}_n  , \hspace{1em} n \in \mathbb{N} \label{equation: empirical limit functional}  
\end{align}
where $\widehat{\mu}_n = L^{-1} \sum_{t=1}^L \widehat{\mu}_t^{(1)} - \widehat{\mu}_t^{(2)}$, $\mu = L^{-1} \sum_{t=1}^L \mu^{(1)}_t - \mu^{(2)}_t$ and 
\begin{align*}
& \widehat{\mu}_t^{(1)} = \frac{1}{n} \sum_{i=1}^n \mathbb{E}_{\widetilde{X}_t \sim q_t} \phi \left ( \log q_t \left ( \widetilde{X}_t \mid X_{<t}^{(i)} \right ) \right ) \hspace{1em} \\
& \widehat{\mu}_t^{(2)} = \frac{1}{n} \sum_{i=1}^n \phi \left ( \log q_t \left ( X_t^{(i)} \mid X_{<t}^{(i)} \right ) \right ) \\
& \mu_t^{(1)} =\mathbb{E}_{X_{<t} \sim p} \mathbb{E}_{X_t\sim q_t} \phi \left ( \log q_t \left ( X_t \mid X_{<t} \right ) \right ) \\
& \mu_t^{(2)} = \mathbb{E}_{X_{<t} \sim p} \mathbb{E}_{X_t\sim p_t} \phi \left ( \log q_t \left ( X_t \mid X_{<t} \right ) \right )
\end{align*}
for each $t = 1, \dots, L$. Similarly, define $\widehat{\Sigma}_n = L^{-1} \sum_{t=1}^L \widehat{\Sigma}_t$ and $\Sigma = L^{-1} \sum_{t=1}^L \Sigma_t$ where 
\begin{align*}
& \widehat{\Sigma}_t = \frac{1}{n} \sum_{t=1}^n \mathbb{E}_{\widetilde{X}_t \sim q_t} \left [ \phi \left ( \log q_t \left ( \widetilde{X}_t \mid X_{<t}^{(i)} \right ) \right ) \phi \left ( \log q_t \left ( \widetilde{X}_t \mid X_{<t}^{(i)} \right ) \right )^\top \right ] - \widehat{\mu}_t^{(1)} \left ( \widehat{\mu}_t^{(1)} \right )^\top
\end{align*}
for each $t = 1, \dots, L$. 

\begin{lemma}\label{lem:bound for beta}
Grant the assumptions in Section \ref{sec:assumptions} hold. Let $\beta^*$ be the maximizer of the function \eqref{equation: limit functional} over all $\beta$'s with $\ell_2$ norm equal to $1$ and let $\widehat{\beta}$ be the maximizer of the empirical counterpart~\eqref{equation: empirical limit functional}. There are absolute constants $\kappa_1$ and $\kappa_2$ depending only on \emph{the constants stated in the assumptions} such that for any $z > 0$ it holds that $ \| \widehat{\beta} - \beta^* \|_2 \leq z$ with probability at least $1 - \kappa_1 \exp \left ( - \kappa_2 d^{-5\gamma} n (\min \{z,1\})^2 \right )$. 
\label{lemma: beta l2 bound}
\end{lemma}
 
\begin{proof}
Observing the $\widehat{\beta} \in \argmax\limits_{\beta} \widehat{Q}_n(\beta)$ and $\beta^* \in \argmax\limits_{\beta} Q^*(\beta)$ for any $z > 0$ 
\begin{align}
\mathbb{P} \left ( \left \| \widehat{\beta} - \beta^* \right \|_2 > z  \right ) 
= & \mathbb{P} \left ( \sup_{\beta : \left \| \beta - \beta^* \right \|_2 > z} Q_n \left ( \beta \right ) - Q_n \left ( \beta^* \right ) > 0 \right ) \label{equation: beta-2 bound} \\
\leq & \mathbb{P} \left ( \sup_{\beta : \left \| \beta - \beta^* \right \|_2 > z} \left | Q_n \left ( \beta \right ) - Q_* \left ( \beta \right ) \right | > \frac{1}{2} \inf_{\beta : \left \| \beta - \beta^* \right \|_2 > z} \left[ Q_* \left ( \beta^* \right ) - Q_* \left ( \beta \right ) \right] \right ) \nonumber \\
& + \mathbb{P} \left ( \left | Q_n \left ( \beta^* \right ) - Q_* \left ( \beta^* \right ) \right | > \frac{1}{2} \inf_{\beta : \left \| \beta - \beta^* \right \|_2 > z} \left[ Q_* \left ( \beta^* \right ) - Q_* \left ( \beta \right ) \right] \right ). \nonumber
\end{align}
For $\beta$ such that $\left \| \beta - \beta^* \right \|_2 > z$, due to Assumption \ref{assumption: eigenvalue condition} with some absolute $\kappa_1$ we must have that
\begin{align}
Q_* \left ( \beta^* \right ) - Q_* \left ( \beta \right ) & \geq \kappa_1 d^{-\gamma} \left \| \beta - \beta^* \right \|_2 \geq \kappa_1 z d^{-\gamma}. 
\label{equation: beta dist lower bound}
\end{align}
For any $\beta \in \mathbb{R}^d$ with $\left \| \beta \right \|_2 = 1$ it holds that
\begin{subequations}
\begin{align}
& \left | Q_n \left ( \beta \right ) - Q_* \left ( \beta \right ) \right | \nonumber \\
\leq& \left ( \beta^\top \Sigma \beta \right ) ^{-\frac{1}{2}} \left | \beta^\top \left ( \widehat{\mu}_n - \mu \right ) \right | + \left | \beta^\top \widehat{\mu}_n \right | \left | \left ( \beta^\top \widehat{\Sigma}_n \beta \right )^{-\frac{1}{2}} - \left ( \beta^\top \Sigma \beta \right )^{-\frac{1}{2}} \right | \nonumber \\
\leq& \left ( \beta^\top \Sigma \beta \right ) ^{-\frac{1}{2}} \left | \beta^\top \left ( \widehat{\mu}_n - \mu \right ) \right | \nonumber \\
& \hspace{1em}+ \frac{1}{2} \left \{ \min \left ( \beta^\top \widehat{\Sigma}_n \beta,  \beta^\top \Sigma \beta \right ) \right \}^{-\frac{3}{2}} \left | \beta^\top \widehat{\mu}_n \right | \left | \beta^\top \widehat{\Sigma}_n \beta -  \beta^\top \Sigma \beta \right | \label{equation: Q bound i} \\
\leq& \left \{ \lambda_{\text{min}} ( \Sigma ) \right \} ^{-\frac{1}{2}} \left | \beta^\top \left ( \widehat{\mu}_n - \mu \right ) \right | \label{equation: Q bound ii}  \\
& \hspace{1em} + \frac{1}{2} \left \{ \min \left ( \lambda_{\text{min}} ( \widehat{\Sigma}_n ),  \lambda_{\text{min}} ( \Sigma ) \right ) \right \}^{-\frac{3}{2}} \left | \beta^\top \widehat{\mu}_n \right | \left | \beta^\top \widehat{\Sigma}_n \beta -  \beta^\top \Sigma \beta \right |, \label{equation: Q bound iii}
\end{align}
\end{subequations}
where in particular (\ref{equation: Q bound i}) holds due to the inequality
\begin{equation}
\left | \frac{1}{\sqrt{x}} - \frac{1}{\sqrt{y}} \right | = \left | \frac{x-y}{\sqrt{xy}(\sqrt{x}+\sqrt{y})}  \right | \leq \frac{1}{2 \left ( \min(x,y) \right )^\frac{3}{2}} \left | x - y \right |
\label{equation: 1/sqrt bound}
\end{equation}
for $x,y > 0$. For (\ref{equation: Q bound ii}) note that due to the boundedness of $\phi(\bullet)$ for any $\beta$ with $\| \beta \|_2 = 1$ it holds that $\beta^\top \left ( \widehat{\mu}_n - \mu \right )$ is the average of $n$ i.i.d. random variables each bounded in absolute value by a constant which does not depend on $\beta$. Therefore lower bounding $\lambda_{\text{min}}(\Sigma) \geq \kappa_2 d^{-\gamma}$ for some absolute $\kappa_2$ using again the boundedness of $\phi(\bullet)$ and applying Hoeffding's inequality, for any $z > 0$
\begin{equation}
\mathbb{P} \left ( \eqref{equation: Q bound ii} > \frac{1}{2} \kappa_1 z d^{-\gamma} \right ) \leq  \mathbb{P} \left ( \left | \beta^\top \left ( \widehat{\mu}_n - \mu \right ) \right | > \kappa_3 z d^{-\frac{3\gamma}{2}} \right ) \leq 2 \exp \left (-\frac{\kappa_4 n z^2}{d^{3 \gamma}} \right )
\label{equation: mean bound}
\end{equation}
for certain absolute $\kappa_3, \kappa_4$. The following argument will be valid on the event
\begin{equation}
\lambda_\text{min} ( \widehat{\Sigma}_n ) \geq \frac{1}{2} \lambda_\text{min} (\Sigma).  
\label{equation: sigam-hat eigen}
\end{equation}
For (\ref{equation: Q bound iii}) notice first that with $\| \beta \|_2 =1$ the quantity $| \beta^\top \widehat{\mu}_n |$ is almost surely bounded from above by some absolute constant independent of $\beta$ and $d$. Moreover due to the boundedness of $\phi(\bullet)$ it is easy to see that the statistic $| \beta^\top \widehat{\Sigma}_n \beta -  \beta^\top \Sigma \beta |$ is a self-bounding function of $n$ random variables with constants (see equation (\ref{equation: bounded difference property})) $c_i \propto \frac{1}{n}$ for $i = 1, \dots, n$ which again do not depend on $\beta$. Therefore, on the event (\ref{equation: sigam-hat eigen}) applying Theorem \ref{theorem: bounded difference} we obtain that
\begin{equation}
\mathbb{P} \left ( \eqref{equation: Q bound iii} > \frac{1}{2} \kappa_1 d^{-\gamma} \right ) \leq \mathbb{P} \left ( \left | \beta^\top \widehat{\Sigma}_n \beta -  \beta^\top \Sigma \beta \right | > \kappa_5 z d^{-\frac{5 \gamma}{2}} \right ) \leq \exp \left ( - \frac{\kappa_6 n z^2}{d^{5 \gamma}} \right )
\label{equation: cov bound}
\end{equation}
for certain absolute $\kappa_5, \kappa_6$. Finally, note that
\begin{equation}
\lambda_{\text{min}} (\widehat{\Sigma}) = \min_{\beta : \| \beta \|_2 = 1} \beta^\top \widehat{\Sigma}_n \beta \geq \lambda_{\text{min}} (\Sigma) - \max_{\beta : \| \beta \|_2 = 1} \beta^\top ( \widehat{\Sigma}_n - \Sigma) \beta 
\label{equation: eign bound}
\end{equation}
and arguing as in (\ref{equation: cov bound}), the final term (\ref{equation: eign bound}) is no larger than $\frac{1}{2} \kappa_2 d^{-\gamma}$ with probability at least 
\begin{equation*}
1 - 2 \exp \left ( - \frac{\kappa_7 n}{d^{2\gamma}} \right ). 
\end{equation*}
Since by Assumption \ref{assumption: eigenvalue condition} we must have that $ \kappa_2 d^{-\gamma} \leq \lambda_{\text{min}} (\Sigma)$ the event (\ref{equation: sigam-hat eigen}) must hold with the above probability. Since the above arguments hold for any $\beta$ with $\| \beta \|_2 = 1$, plugging (\ref{equation: mean bound}) and (\ref{equation: cov bound}) back into (\ref{equation: beta-2 bound}) and accounting for the event (\ref{equation: eign bound}) the stated result follows. 
\end{proof}

\subsection{Proof of Theorem \ref{thm:TNRlowerbound}}
According to the decomposition $T_w(\bm{X})=T_w^{(1)}(\bm{X})-T_w^{(2)}(\bm{X})$ with $T_w^{(1)}(\bm{X}),T_w^{(2)}(\bm{X})$ defined by
\begin{eqnarray}\label{eq:Tw-decomposition}
   T_w^{(1)}(\bm{X}) &=&\frac{\sum_t [w(\log q_t(X_t|X_{<t}))-\mathbb{E}_{\widetilde{X}_t \sim p_t} w(\log q_t(\widetilde{X}_t|X_{<t}))]}{\sqrt{\sum_t \textrm{Var}_{\widetilde{X}_t \sim q_t} (w(\log q_t(\widetilde{X}_t|X_{<t})))}}\nonumber\\
    T_w^{(2)}(\bm{X})&=&\frac{\sum_t [\mathbb{E}_{\widetilde{X}_t \sim q_t} w(\log q_t(\widetilde{X}_t|X_{<t}))-\mathbb{E}_{\widetilde{X}_t \sim p_t} w(\log q_t(\widetilde{X}_t|X_{<t})))]}{\sqrt{\sum_t \textrm{Var}_{\widetilde{X}_t \sim q_t} (w(\log q_t(\widetilde{X}_t|X_{<t})))}},
\end{eqnarray}
 we obtain that the TNR can be represented as
\begin{eqnarray}
    \mathbb{P}_{\bm{X}\sim p}\left( T_w(\bm{X}) \leq z_\alpha \right) &=& \mathbb{P}_{\bm{X}\sim p}\left( T_w^{(1)}(\bm{X}) \leq z_\alpha + T_w^{(2)}(\bm{X}) \right)
\end{eqnarray}
It is easy to verify that when $\bm{X} \sim p$, $T_w^{(1)}(\bm{X}) \sigma_{q,L}/\sigma_{p,L}$ converges to standard normal distribution. Specifically, using Lemma \ref{thm: convergence-rate-MCLT}, we obtain that 
\begin{eqnarray}
    \mathbb{P}_{\bm{X}\sim p}\left( T_w(\bm{X}) \leq z_\alpha \right)  &=& \mathbb{P}_{\bm{X}\sim p}\left( T_w^{(1)}(\bm{X})\frac{\sigma_{q,L}}{\sigma_{p,L}} \leq (z_\alpha + T_w^{(2)}
    (\bm{X}))\frac{\sigma_{q,L}}{\sigma_{p,L}} \right)\nonumber\\
    &\geq& \Phi(z_\alpha + T_w^{(2)}(\bm{X})) + \left(\Phi\left((z_\alpha + T_w^{(2)}
    (\bm{X}))\frac{\sigma_{q,L}}{\sigma_{p,L}}\right) - \Phi(z_\alpha + T_w^{(2)}(\bm{X})) \right)\nonumber\\
    &&\qquad +O\left(\log L/\sqrt{L}\right)+o_p(1)\nonumber\\
    &\geq& \Phi(z_\alpha + T_w^{(2)}(\bm{X})) - \sup_{z\in\mathbb{R}} \Phi'(z) \times \left| z_\alpha + T_w^{(2)}(\bm{X}) \right| \times \left| \frac{\sigma_{q,L}}{\sigma_{p,L}}-1 \right| \nonumber\\
    &&\qquad +O\left(\log L/\sqrt{L}\right)+o_p(1),\nonumber
\end{eqnarray}
where the little-$o_p$ term arises due to the asymptotic equivalence between $\sigma_{p,L}$ and $\bar{\sigma}_{p,L}$ in Assumption \ref{ass:variance-ratio}.

Take expectation on both sides, we have by Assumption~\ref{ass:equal-variance} that
\begin{eqnarray}
    \mathbb{P}_{\bm{X}\sim p}\left( T_w(\bm{X}) \leq z_\alpha \right) \geq \mathbb{E}\Phi(z_\alpha + T_w^{(2)}(\bm{X})) + o(1) + O(\log L/ \sqrt{L}).\nonumber
\end{eqnarray}
Next, define $\widetilde\sigma_{q,L}^2 = \mathbb{E}_{\bm{X}\sim p}\sigma_{q,L}^2$. It follows that $T_w^{(2*)}(\bm{X})= \mathbb{E}\left\{T_w^{(2)}(\bm{X})\frac{\sigma_{q,L}}{\widetilde{\sigma}_{q,L}}\right\}$. Under the equal variance assumption in Assumption \ref{ass:equal-variance}, we also have $ \sigma_{q,L} - \widetilde\sigma_{q,L} \to 0$ in probability. It follows that for any $\epsilon >0$,
\begin{eqnarray}
    \mathbb{E}\Phi(z_\alpha + T_w^{(2)}(\bm{X})) &=& \mathbb{E}\Phi(z_\alpha + T_w^{(2)}(\bm{X}))\mathbb{I}\{|\sigma_{q,L} - \widetilde\sigma_{q,L}| \leq\epsilon\} \nonumber\\
    &&\qquad + \mathbb{E}\Phi(z_\alpha + T_w^{(2)}(\bm{X}))\mathbb{I}\{|\sigma_{q,L} - \widetilde\sigma_{q,L}| >\epsilon\} \nonumber\\
    &\geq& \mathbb{E}\Phi(z_\alpha + T_w^{(2)}(\bm{X}))\mathbb{I}\{|\sigma_{q,L} - \widetilde\sigma_{q,L}| \leq\epsilon\} \nonumber\\%- \mathbb{P}(|\sigma_{q,L} - \widetilde\sigma_{q,L}| >\epsilon)
    &\geq & \mathbb{E}\Phi\left(z_\alpha + T_w^{(2)}(\bm{X})\frac{\sigma_{q,L}}{\widetilde\sigma_{q,L} + \text{sgn}(T_w^{(2)})\epsilon}\right)\mathbb{I}\{|\sigma_{q,L} - \widetilde\sigma_{q,L}| \leq\epsilon\} \nonumber\\
    &\geq& \mathbb{E}\Phi\left(z_\alpha + T_w^{(2)}(\bm{X})\frac{\sigma_{q,L}}{\widetilde\sigma_{q,L} + \text{sgn}(T_w^{(2)})\epsilon}\right) \nonumber\\
    &&\qquad -  \mathbb{E}\Phi\left((z_\alpha + T_w^{(2)}(\bm{X}))\frac{\sigma_{q,L}}{\widetilde\sigma_{q,L} + \text{sgn}(T_w^{(2)})\epsilon}\right)\mathbb{I}\{|\sigma_{q,L}- \widetilde\sigma_{q,L}| >\epsilon \} \nonumber\\
    &\geq& \mathbb{E}\Phi\left((z_\alpha + T_w^{(2)}(\bm{X}))\frac{\sigma_{q,L}}{\widetilde\sigma_{q,L} +\text{sgn}(T_w^{(2)}) \epsilon}\right) - \mathbb{P}(|\sigma_{q,L} - \widetilde\sigma_{q,L}| >\epsilon), \nonumber
\end{eqnarray}
where the first inequality is obtained due to $\Phi$ is non-negative and the second inequality holds due to the monotonicity and boundedness of $\Phi$.  Together with Lemma \ref{lem:Taylorexpansion-of-normal-cdf} and Assumption \ref{ass:equal-variance}, we obtain
\begin{equation}\label{eqn:thm1-final-ineq}
\begin{split}
    &\mathbb{P}_{\bm{X}\sim p}\left( T_w(\bm{X}) \leq z_\alpha \right)
    \\
    \geq& \min\left\{1-\alpha, \alpha + \phi(z_\alpha)\mathbb{E}\left\{T_w^{(2)}(\bm{X})\frac{\sigma_{q,L}}{\widetilde{\sigma}_{q,L}}\right\} \right\} \frac{\widetilde{\sigma}_{q,L}}{\widetilde{\sigma}_{q,L} + \text{sgn}(T_w^{(2)})\epsilon} \\
    &-\mathbb{P}\{|\sigma_{q,L} - \widetilde\sigma_{q,L}| \geq\epsilon\} +O\left(\log L/\sqrt{L}\right) + o(1).
\end{split}
\end{equation}
Let $L\to \infty$ and using the fact that $\mathbb{E}\left\{T_w^{(2)}(\bm{X})\frac{\sigma_{q,L}}{\widetilde{\sigma}_{q,L}}\right\}=  T_w^{(2*)}(\bm{X})$, we obtain that TNR is asymptotically lower bounded by $\min\{1-\alpha, \alpha + \phi(z_\alpha)T_w^{(2*)}(\bm{X}) \}\frac{\widetilde{\sigma}_{q,L}}{\widetilde{\sigma}_{q,L} + \text{sgn}(T_w^{(2)})\epsilon}$. By taking $\epsilon \to 0$, then the conclusion of Theorem \ref{thm:TNRlowerbound} follows.

\begin{remark}
    In fact, the equal variance condition (Assumption ~\ref{ass:equal-variance}) can be relaxed. Specifically, it is not necessary for the two variance $\sigma_{q,L}^2$ and $\sigma_{p,L}^2$ to be asymptotically equivalent in probability. Rather, it suffices to require their ratio to converge to some positive constant in probability, i.e.,
    \begin{equation*}
        \frac{\sigma_{q,L}}{\sigma_{p,L}}\stackrel{P}{\rightarrow} K_0.
    \end{equation*}
    Since $K_0$ need not be $1$, this proportionality condition is considerably weaker than the equal variance assumption and is more likely to hold in practice. Under this relaxed condition, following nearly identical arguments to those in Theorem \ref{thm:TNR}, we can show that TNR is asymptotically lower bounded by:
    \begin{equation*}
        \min \left\{1-\Phi(K_0z_\alpha), \Phi(K_0z_\alpha)+\phi(K_0z_\alpha)K_0T_w^{(2*)} \right\}.
    \end{equation*}
    This lower bound differs from the one in Theorem ~\ref{thm:TNR} due to the change in assumptions. However, since $\phi(K_0z_\alpha)$ depends solely on $\alpha$ (not on the witness function $w$), our core conclusion -- maximizing the lower bound is equivalent to maximizing $T_w^{(2*)}$ -- remains valid. Thus, our proposed methodology remains theoretically sound.
\end{remark}

\begin{comment}
\begin{remark}
    It is worth noting that under additional stochastic dominance condition (Assumption ~\ref{ass:stochastic-dominance}), we have $T_w^{(2)}(\bm{X})\sigma_{q,L} > 0$. Then, a sharper lower bound can be obtained by applying Lemma \ref{lem:sharp-Taylorexpansion-of-normal-cdf} in the last step (inequality \eqref{eqn:thm1-final-ineq}), then follow a similar argument, we will obtain that TNR is asymptotically lower bounded by
    \begin{eqnarray}\label{eqn:sharper-TNR-lower-bound}
        \sup_{0<\beta\leq \alpha}\min\{1-\beta, \alpha + \phi(z_\beta)T_w^{(2*)}(\bm{X}) \}.%\frac{\widetilde{\sigma}_{q,L}}{\widetilde{\sigma}_{q,L} - \epsilon}
    \end{eqnarray}
    Noted that with any fixed $\alpha \in (0, 1/2)$, the lower bound in \eqref{eqn:sharper-TNR-lower-bound} may tend to $1$ given that $T_w^{(2*)}(\bm{X})$ is sufficiently large, which is a sharper bound than in Theorem \ref{thm:TNRlowerbound}.
\end{remark}
\end{comment}

\subsection{Proof of Theorem~\ref{thm:FNR}}
\begin{proof}
Denote $Z_t = \widehat{w}\left(\log q_t(X_t|X_{<t})\right) - \mathbb{E}_{\widetilde{X}_t \sim q_t(\bullet|X_{<t})}\widehat{w}\left(\log q_t(X_t|X_{<t})\right)$. Then if $\bm{X}\sim q$, we have $\mathbb{E}\left\{Z_t|X_{<t}\right\} = 0$ almost surely. Without loss of generality, assume $\widehat{w}$ is bounded. Otherwise, we can define $\widehat{w}=\phi^\top \widehat{\beta}/\|\widehat{\beta}\|_2$ to make it bounded. Under the lower bound assumption in Assumption \ref{ass:equal-variance}, it is easy to verify that $Z_t$ satisfies all conditions of Lemma~\ref{thm: convergence-rate-MCLT}. Therefore, by invoking Lemma \ref{thm: convergence-rate-MCLT}, we obtain for any $\alpha \in (0,1)$,
\begin{eqnarray}
    \text{FNR}_{\widehat{w}} - \alpha &=& \mathbb{P}_{\bm{X}\sim q}(T_{\widehat{w}}(\bm{X})\leq z_\alpha) - \Phi(z_\alpha)\nonumber\\
    &=& \mathbb{P}_{\bm{X}\sim q}\left(\frac{\sum_{t=1}^L Z_t}{\sum_{t=1}^L \mathbb{E}\{ Z_t^2|X_{<t}\}}\leq z_\alpha\right) - \Phi(z_\alpha)\nonumber\\
    &=& O\left(\frac{\log L}{\sqrt{L}}\right) + O\big((\mathbb{E}\big| V_L -1\big|)^{1/3}\big). \nonumber
\end{eqnarray}
Taking expectation on both sides, we obtain
\begin{equation*}
    \mathbb{E}(\text{FNR}_{\widehat{w}}) = \alpha + O\left(\frac{\log L}{\sqrt{L}}\right) + O\big((\mathbb{E}\big| V_L -1\big|)^{1/3}\big).
\end{equation*}
This completes the proof.
\end{proof}

\subsection{Proof of Theorem \ref{thm:TNR}}

\begin{proof}
Since $\mathbb{E} (\text{TNR}_{\widehat{w}}) \geq \text{TNR}_{w^*} - \mathbb{E} (|\text{TNR}_{\widehat{w}} - \text{TNR}_{w^*}|)$, it is enough to upper bound the second term in the last expression. Denote by $\widehat{T}_n (\bullet)$ and $\widehat{T}^* (\bullet)$ respectively the classifier (\ref{eqn:Tw}) using witness functions $\widehat{w}(\bullet) =  \phi (\bullet)^\top \widehat{\beta}$ and $w^*(\bullet) =  \phi (\bullet)^\top \beta^*$ (see the definition of $\widehat{\beta}$ and $\beta^*$ in Lemma~\ref{lem:bound for beta}). Write $\widehat{\Delta}_n (\boldsymbol{x}) = |\widehat{T}_n (\boldsymbol{x}) - T^* (\boldsymbol{x})|$ and $\widehat{\Delta}_n = \sup_{\boldsymbol{x}} \Delta_n (\boldsymbol{x})$. For any $z_\alpha > 0$ we have that
\begin{align}
\left | \text{TNR}_{\widehat{w}} - \text{TNR}_{w^*} \right | = & \left | \mathbb{P}_{\boldsymbol{X} \sim p} \left ( \widehat{T}_n \left ( \boldsymbol{X} \right ) \leq z_\alpha \right ) - \mathbb{P}_{\boldsymbol{X} \sim p} \left ( T^* \left ( \boldsymbol{X} \right ) \leq z_\alpha \right ) \right | \nonumber \\
=& \left | \int  \boldsymbol{1}_{\left \{ \widehat{T}_n \left ( \boldsymbol{x} \right ) \leq z_\alpha \right \}} - \boldsymbol{1}_{\left \{ T^* \left ( x \right ) > z_\alpha \right \}} \mathrm{d} p \left ( \boldsymbol{x} \right ) \right | \nonumber \\
\leq&  \int \left | \boldsymbol{1}_{\left \{ \widehat{T}_n \left ( \boldsymbol{x} \right ) \leq z_\alpha \right \}} - \boldsymbol{1}_{\left \{ T^* \left ( x \right ) > z_\alpha \right \}} \right | \mathrm{d} p \left ( \boldsymbol{x} \right ) \nonumber \\
=& \int \boldsymbol{1}_{\left \{ \widehat{T}_n (\boldsymbol{x}) \leq z_\alpha, T^* (\boldsymbol{x}) > z_\alpha \right \}} + \boldsymbol{1}_{\left \{ \widehat{T}_n (\boldsymbol{x}) > z_\alpha, T^* (\boldsymbol{x}) \leq z_\alpha \right \}} \mathrm{d} p \left ( \boldsymbol{x} \right ) \nonumber \\
\leq& 2 \int \boldsymbol{1}_{\left \{ |T^*(\boldsymbol{x}) - z_\alpha| \leq \widehat{\Delta}_n \right \}} \mathrm{d} p \left ( \boldsymbol{x} \right ) \nonumber \\
=& 2 \mathbb{P}_{\boldsymbol{X} \sim p} \left ( \left | T^* (\boldsymbol{X}) - z_\alpha \right | \leq \widehat{\Delta}_n \right ). 
\label{equation: TNR bound}
\end{align}
Due to Assumption \ref{assumption: margin condition} on the event
\begin{equation}
\left \{ \widehat{\Delta}_n \leq \delta_0 \right \} 
\label{equation: Delta-hat small}
\end{equation}
we will have that $\left | \text{TNR}_{\widehat{w}} - \text{TNR}_{w^*} \right | \leq \kappa_3 \widehat{\Delta}_n$ for some absolute $\kappa_3$. We therefore focus on bounding the quantity $\widehat{\Delta}_n$. For each $w \in \Omega$ and each $j = 1, \dots, L$, we introduce the quantities:
\begin{align*}
& Y_j^{(w)} = w \left ( \log q_j \left ( X_j \mid X_{<j} \right ) \right ), \\
& \mu_j^{(w)} = \mathbb{E}_{\widetilde{X}_j \sim q_j}  w \left ( \log q_j \left ( \widetilde{X}_j \mid X_{<j} \right ) \right ), \\
& (\sigma_j^{(w)})^2 =  \textup{Var}_{\widetilde{X}_j \sim q_j} w \left ( \log q_j \left ( \widetilde{X}_j \mid X_{<j} \right ) \right ). 
\end{align*}
With this notation in place we have that for any $\boldsymbol{x}$
\begin{subequations}
\begin{align}
\widehat{\Delta}_n (\boldsymbol{x}) & \leq \frac{1}{\sqrt{L}}  \left | \sum_{j=1}^L y_j^{(\widehat{w})} - \mu_j^{(\widehat{w})} \right | \left | \left [ L^{-1} \sum_{j=1}^L \left ( \sigma_j^{(\widehat{w})} \right )^2 \right ] ^ {-1/2} - \left [ L^{-1} \sum_{j=1}^L \left ( \sigma_j^{(w^*)} \right )^2 \right ] ^{-1/2} \right | \label{equation: Delta bound 1} \\
& \hspace{2em} + \left \{ \frac{1}{L} \sum_{j=1}^L \left ( \sigma_j^{(w^*)} \right )^2 \right \}^{-\frac{1}{2}} \frac{1}{\sqrt{L}} \left | \sum_{j=1}^L \left ( Y_j^{(\widehat{w})} - Y_j^{(w^*)} \right ) - \left ( \mu_j^{(\widehat{w})} - \mu_j^{(w^*)} \right ) \right |, \label{equation: Delta bound 2}
\end{align}
\end{subequations}
where for clarity we have suppressed dependence on $\boldsymbol{x}$ above. For ease of notation put $Z_t = \log q_t ( X_t \mid X_{<t} )$ and $\widetilde{Z}_t = \log q_t ( \widetilde{X}_t \mid X_{<t} )$ where $\widetilde{X}_t \sim q_t$. Write also $\phi (\bullet) = (B_1 (\bullet), \dots, B_d (\bullet))^\top$. Recalling that $w(\bullet) = \phi (\bullet)^\top \beta$ for arbitrary $j = 1, \dots, L$ we have
\begin{align*}
\left | \left ( \sigma_j^{(\widehat{w})} \right )^2 - \left ( \sigma_j^{(w^*)} \right )^2 \right | & \leq \mathbb{E} \left [ \sum_{l_1=1}^d \sum_{l_2=1}^d \left | \widehat{\beta}_{l_1} \widehat{\beta}_{l_2} - \beta_{l_1}^* \beta_{l_2}^* \right | \left ( \left | B_{l_1} (Z_j) B_{l_2} (z_j) \right | \right . \right . \\
& \hspace{4em} + \left | \mathbb{E} \left [ B_{l_1} (\widetilde{Z}_j) \right ] \mathbb{E} \left [ B_{l_2} (\widetilde{Z}_j) \right ] \right | + 2 \left | B_{l_1} (z_j) \mathbb{E} \left [ B_{l_2} (\widetilde{Z}_j) \right ] \right | \left . \left . \right ) \right ] \\ 
& \leq \frac{\kappa_4}{2} \sum_{l_1=1}^d \sum_{l_2=1}^d \left | \widehat{\beta}_{l_1} \widehat{\beta}_{l_2} - \beta_{l_1}^* \beta_{l_2}^* \right | \\
& = \frac{\kappa_4}{2} \sum_{l_1=1}^d \sum_{l_2=1}^d \left | \widehat{\beta}_{l_1} \left ( \widehat{\beta}_{l_2} - \beta_{l_2}^* \right ) - \beta_{l_2}^* \left ( \beta_{l_1}^* - \widehat{\beta}_{l_1} \right ) \right | \\
& \leq \frac{\kappa_4}{2} \sum_{l_1=1}^d \sum_{l_2=1}^d \left | \widehat{\beta}_{l_1} \right | \left | \widehat{\beta}_{l_2} - \beta_{l_2}^* \right | + \frac{\kappa_4}{2} \sum_{l_1=1}^d \sum_{l_2=1}^d \left | \beta_{l_2}^* \right | \left | \widehat{\beta}_{l_1} - \beta_{l_1}^* \right | \\
& = \kappa_5 \sqrt{d}  \left \| \widehat{\beta} - \beta^* \right \|_1 \kappa_5 \\
& \leq d \left \| \widehat{\beta} - \beta^* \right \|_2,  
\end{align*}
for absolute $\kappa_4, \kappa_5$. Consequently, using inequality (\ref{equation: 1/sqrt bound}) and Assumption \ref{assumption: eigenvalue condition}, on the event (\ref{equation: eign bound}) we obtain that with absolute $\kappa_6$:
\begin{equation}
(\ref{equation: Delta bound 1}) \leq \kappa_6 d^3 \frac{1}{\sqrt{L}}  \left | \sum_{j=1}^L y_j^{(\widehat{w})} - \mu_j^{(\widehat{w})} \right | \left \| \widehat{\beta} - \beta^* \right \|_2. 
\label{equation: Delat bound 1-i}
\end{equation}
Observe that conditional on $\widehat{\beta}$ the term $\sum_{j=1}^L (y_j^{(\widehat{w})} - \mu_j^{(\widehat{w})})$ is a martingale with increments bounded from above almost surely by a constant independent on $\widehat{\beta}$; by the  Azuma–Hoeffding inequality the normalized sum in (\ref{equation: Delat bound 1-i}) has sub-Gaussian tails. By Lemma \ref{lemma: beta l2 bound} the term $\| \widehat{\beta} - \beta^* \|_2$ likewise has sub-Gaussian tails. Therefore on the relevant events we obtain that \eqref{equation: Delat bound 1-i} has sub-exponential tails, and consequently for any $z > 0$
\begin{equation}
\mathbb{P} \left ( \ref{equation: Delat bound 1-i} > z \right ) \leq \kappa_7 \exp \left ( -\kappa_8 \min \left \{ z^2 \frac{n}{d^{5\gamma+6}}, z \sqrt{\frac{n}{d^{5\gamma+6}}} \right \} \right )
\label{equation: Delta tail behaviour}
\end{equation}
for certain absolute $\kappa_7, \kappa_8$. Similar arguments show that the normalized sum in \eqref{equation: Delta bound 2} has the same tail behavior as \eqref{equation: Delta tail behaviour}. Since the above augments do not depend on $\boldsymbol{x}$ we obtain that \eqref{equation: Delta tail behaviour} likewise described the tail behavior of $\widehat{\Delta}_n$. Consequently, on the relevant events we obtain that
\begin{align}
\mathbb{E} \left | \text{TNR}_{\widehat{w}} - \text{TNR}_{w^*} \right | & = \int_0^\infty \mathbb{P} \left ( \left | \text{TNR}_{\widehat{w}} - \text{TNR}_{w^*} \right |  > z \right ) \mathrm{d}z \nonumber \\
& \leq \int_0^\infty \mathbb{P} \left ( \kappa_3 \widehat{\Delta}_n > z \right ) \mathrm{d} z \nonumber \\
& \leq \kappa_9 \int_0^\infty \exp \left ( -\kappa_{10} \min \left \{ z^2 \frac{n}{d^{5\gamma+6}}, z \sqrt{\frac{n}{d^{5\gamma+6}}} \right \} \right ) \mathrm{d} z \nonumber \\
& \leq \kappa_{11} \sqrt{\frac{d^{5\gamma+6}}{n}} 
\label{equation: TNR good event bound}
\end{align}
for certain absolute $\kappa_9, \kappa_{10}, \kappa_{11}$. When the events \eqref{equation: Delta-hat small} and \eqref{equation: 1/sqrt bound} do not hold from \eqref{equation: TNR bound} we have the conservative bound $\mathbb{E} \left | \text{TNR}_{\widehat{w}} - \text{TNR}_{w^*} \right | \leq 1$. However, the probability of these events not holding is smaller than \eqref{equation: TNR good event bound} up to constants. Therefore, the stated result follows by the law of total expectation. 
\end{proof}

\subsection{Smooth witness functions} \label{sec: smooth witness function}

An important quality of spline estimators, which in part motivated our choice of estimator for the witness function, is their ability to learn smooth regression functions at optimal rates. Following the proof of Theorem 2.1 in \cite{chen2015optimal} one can show that when the optimal witness (that is, the function minimizing the functional \eqref{equation: limit functional} among all functions with bounded $\ell_2$ norm) is $\beta$-H{\"o}lder smooth, the estimated witness function attains the rate
\begin{align*}
\sup_z \left | \widehat{w} (z) - w^*(z) \right | = O_P \left ( \sqrt{\frac{d \log (n)}{n}} \right ) + O(d^{-\beta}). 
\end{align*}
Consequently, choosing the number of spline bases as $d = \Theta ( (n / \log(n))^{\frac{1}{2\beta+1}} )$ the sup norm loss will be of the order $O_P \left ( (\log(n) / n )^\frac{\beta}{2 \beta + 1} \right )$. This rate was shown to be optimal by \cite{stone1982optimal}. One can therefore show that choosing the number of spline bases in this way the expected true negative rate will be lower bounded as 
\begin{eqnarray*}
\mathbb{E}(\text{TNR}_{\widehat{w}}) \geq \text{TNR}_{w^*} - O \left ( (\log(n) / n)^\frac{\beta}{2 \beta + 1} \right ). 
\end{eqnarray*}
In this section we provide a sketch of this result. 

We argue along the same lines as the proof of Theorem \ref{thm:TNR}. Since $\mathbb{E} (\text{TNR}_{\widehat{w}}) \geq \text{TNR}_{w^*} - \mathbb{E} (|\text{TNR}_{\widehat{w}} - \text{TNR}_{w^*}|)$ is enough to upper bound $\mathbb{E} (|\text{TNR}_{\widehat{w}} - \text{TNR}_{w^*}|)$. Define the quantity $\widehat{\Delta}_n = \sup_{\boldsymbol{x}} |\widehat{T}_n (\boldsymbol{x}) - T^* (\boldsymbol{x})|$, and introduce the event 
\begin{equation}
A(\kappa_1) = \left \{ \widehat{\Delta}_n \leq \kappa_1 \times (d^{-\beta} + \left ( \log (n) / n  \right )^{\frac{\beta}{2 \beta + 1}}) \right \}
\label{equation: A kappa event}
\end{equation}
Therefore, with arbitrary $\kappa_1$ for some absolute $\kappa_2$ we have that
\begin{subequations}
\begin{align}
\mathbb{E} (|\text{TNR}_{\widehat{w}} - \text{TNR}_{w^*}|) & = \mathbb{E} (|\text{TNR}_{\widehat{w}} - \text{TNR}_{w^*}| \times \mathbf{1}_{\{ A(\kappa_1) \}}) + \mathbb{E} (|\text{TNR}_{\widehat{w}} - \text{TNR}_{w^*}| \times \mathbf{1}_{\{ \neg A(\kappa_1) \}}) \nonumber \\ 
& \leq 2 \mathbb{E} ( \mathbb{P}_{\boldsymbol{X} \sim p}  ( | T^* (\boldsymbol{X}) - z_\alpha | \leq \widehat{\Delta}_n ) \times \mathbf{1}_{\{ A(\kappa_1) \}} ) + \mathbb{P} (\neg A(\kappa_1)) \label{equation: new TNR bound (i)} \\
& \leq \kappa_1 \mathbb{E} (\hat{\Delta}_n \mid  A(\kappa_1) ) + \mathbb{P} (\neg  A(\kappa_1)) \label{equation: new TNR bound (ii)}
\end{align}
\end{subequations}
where \eqref{equation: new TNR bound (i)} follows from the arguments leading up to \eqref{equation: TNR bound} and \eqref{equation: new TNR bound (ii)} holds on the event \eqref{equation: Delta-hat small}. Using the mean value theorem and Assumption \ref{ass:equal-variance} one can show that up to constants $\hat{\Delta}_n$ is smaller than $\sup_z \left | \widehat{w} (z) - w^*(z) \right |$, therefore by \eqref{equation: A kappa event} the first term in \eqref{equation: new TNR bound (ii)} will be of the order $O(d^{-\beta} + \left ( \log (n) / n  \right )^{\frac{\beta}{2 \beta + 1}})$. Examining the proof of Theorem 2.1 in \cite{chen2015optimal} it can be seen that for any $\kappa_3 > 0$ on may choose $\kappa_1$ so that $\mathbb{P} (\neg  A(\kappa_1)) < n^{-\kappa_3}$. Therefore, we can choose $\kappa_1$ appropriately so that the second term in \eqref{equation: new TNR bound (ii)} will be of the order $o(\left ( \log (n) / n  \right )^{\frac{\beta}{2 \beta + 1}})$. Finally letting $d = \Theta ( (n / \log(n))^{\frac{1}{2\beta+1}} )$ yields the desired result. 

As mentioned in the main text, when the optimal witness function is believed to be smooth but the order of the smoothness is not known the number of spline bases can be chosen in a data driven way via Lepski's method \citep{lepski1997optimal}. 

\section{Additional numerical results}

\subsection{Results on additional open-source models}

\begin{table}[H]
\small
\centering
\caption{Performance on three open-source LLMs (Qwen2.5, Mistral, LLaMA3) across five datasets.}\label{tab:white-box-advanced}
\begin{tabular}{llccccc}
\toprule
\textbf{Model} & \textbf{Method} & \textbf{XSum} & \textbf{Writing} & \textbf{Essay} & \textbf{SQuAD} & \textbf{Yelp} \\
\midrule
\multirow{10}{*}{Qwen2.5}
& Likelihood & 0.6175	& 0.7041	& 0.6755 & 0.5183	& 0.6793	\\
& Entropy & 0.5403	& 0.5043	& 0.5073 & 0.5232	& 0.5236	\\
& LogRank & 0.6325	& 0.7150	& 0.6958 & 0.5166	& 0.6943	\\
& Binoculars & 0.6297	& 0.7578	& 0.8018 & 0.6164	& 0.7199	\\
& TextFluoroscopy & 0.5778	& 0.5110	& 0.5638 & 0.5383	& 0.5060	\\
& RADAR & 0.6469	& 0.6190	& 0.6061 & \textbf{0.6262}	& 0.6276 	\\
& ImBD & 0.6653	& 0.6584	& 0.7874 & 0.5168	& 0.7392 	\\
& BiScope & 0.6320	& 0.6610	& 0.6625 & 0.6250	& 0.7050 	\\
\cdashline{2-7}
& Fast-DetectGPT & 0.7523	& 0.8513	& 0.8347 & 0.5016	& 0.8465	\\
& AdaDetectGPT & \textbf{0.7963}	& \textbf{0.8965}	& \textbf{0.8799} & 0.6044	& \textbf{0.8915}	\\
\cdashline{2-7}
& Relative & 17.7682	& 30.3912	& 27.3167 & 0.6431	& 29.3165	\\
\midrule
\multirow{10}{*}{Mistral}
& Likelihood & 0.7409 & 0.8643 & 0.8667 & 0.7068 & 0.7598 \\
& Entropy & 0.5290 & 0.5420 & 0.6052 & 0.5070 & 0.5103 \\
& LogRank & 0.7270 & 0.8446 & 0.8467 & 0.7041 & 0.7499 \\
& Binoculars & 0.7218 & 0.8440 & 0.8314  & 0.7258 & 0.7502 \\
& TextFluoroscopy & 0.6210 & 0.5555 & 0.5127 & 0.5772 & 0.5109 \\
& RADAR & 0.6518 & 0.6537 & 0.6292 & 0.6055 & 0.6018 \\
& ImBD & 0.7683 & 0.8391 & 0.8631 & 0.8073 & 0.7440 \\
& BiScope & 0.7320 & \textbf{0.8740} & 0.9000 & 0.7283 & 0.7840 \\
\cdashline{2-7}
& Fast-DetectGPT & 0.8922 & 0.8151 & 0.9052 & 0.8812 & 0.8902 \\
& AdaDetectGPT & \textbf{0.8944} & 0.8275 & \textbf{0.9069} & \textbf{0.8851} & \textbf{0.9026} \\
\cdashline{2-7}
& Relative & 2.0423 & 6.6718 & 1.8332 & 3.3051 & 11.2763 \\
\midrule
\multirow{10}{*}{LLaMA3}
& Likelihood & 0.8236 & 0.8929 & 0.9115 & 0.7071 & 0.8915	\\
& Entropy & 0.5545 & 0.5732 & 0.5626 & 0.5047 & 0.5010	\\
& LogRank & 0.8634 & 0.9122 & 0.9351 & 0.7422 & 0.9146	\\
& Binoculars & 0.9546 & 0.9845 & \textbf{0.9949}  & 0.9469 & 0.9854	\\
& TextFluoroscopy & 0.5479 & 0.5274 & 0.5478 & 0.5535 & 0.5362	\\
& RADAR & 0.7154 & 0.7285 & 0.7835 & 0.6619 & 0.7875	\\
& ImBD & 0.8643 & 0.8837 & 0.8928 & 0.7596 & 0.8677	\\
& BiScope & 0.9450 & 0.9830 & 0.9900 & 0.8783 & 0.9860	\\
\cdashline{2-7}
& Fast-DetectGPT & 0.9734 & 0.9879 & 0.9901 & 0.9488 & 0.9882	\\
& AdaDetectGPT & \textbf{0.9782} & \textbf{0.9893} & 0.9924 & \textbf{0.9553} & \textbf{0.9900}	\\
\cdashline{2-7}
& Relative & 18.0119 & 11.6202 & 22.7215 & 12.6288 & 15.7610	\\
\bottomrule
\end{tabular}
\end{table}

\subsection{Additional results on black-box setting}

\begin{table}[H]
\centering
\caption{AUC scores of various detectors to detect text generated by Gemini-2.5 across datasets.}\label{tab:advanced-gemini}
\begin{adjustbox}{width=0.68\textwidth}
\begin{tabular}{l|ccccc}
\toprule
\textbf{Method} & XSum & Writing & Yelp & Essay & Avg. \\
\midrule
RoBERTaBase   & 0.5311 & 0.5202 & 0.5624 & 0.7279 & 0.5854 \\
RoBERTaLarge  & 0.6583 & 0.5888 & 0.6029 & 0.8180 & 0.6845 \\
Likelihood    & 0.7127 & 0.7547 & 0.6566 & 0.7565 & 0.7201 \\
Entropy       & 0.5754 & 0.5088 & 0.6023 & 0.6038 & 0.5726 \\
LogRank       & 0.6084 & 0.5743 & 0.6896 & 0.7504 & 0.6557 \\
LRR           & 0.5960 & 0.5382 & 0.5580 & 0.6703 & 0.5906 \\
Binoculars    & \textbf{0.8500} & 0.9453 & \textbf{0.9698} & 0.9908 & \textbf{0.9390} \\
RADAR         & 0.8184 & 0.5152 & 0.6300 & 0.5891 & 0.6382 \\
% ImBD          & 0.8429 & 0.9623 & 0.9708 & 0.9938 & 0.9424 \\
BiScope       & 0.7633 & 0.6800 & 0.7097 & 0.8167 & 0.7642 \\
\cdashline{1-6}
% Fast-DetectGPT & 0.5768 & 0.7373 & 0.7419 & 0.7878 & 0.7109 \\
Fast-DetectGPT & 0.8404 & 0.9443 & 0.9695 & 0.9914 & 0.9364	\\
AdaDetectGPT  & 0.8432 & \textbf{0.9484} & 0.9644 & \textbf{0.9947} & 0.9377 \\
\cdashline{1-6}
% Relative      & 62.9470 & 80.3570 & 86.2087 & 97.4869 & 78.4355 \\
Relative  (\parbox[c]{1em}{\tikz{\hspace*{3.9pt}\draw[blue, -latex] (0,0) -- (0,0.3);}})    & 1.7544 & 7.4163 & --- & 37.8238 & 1.9916 \\
\bottomrule
\end{tabular}
\end{adjustbox}
\end{table}

\subsection{Computational analysis}\label{sec:computational-analysis}

In this part, we study the runtime for learning the witness function. From Table~\ref{tab:runtime-d}, the runtime of AdaDetectGPT is around 44 seconds and changes marginally with respect to $d$. This is because we can use a closed-form expression to learn a witness function. This time is nearly negligible compared with the time required to compute logits, which involves feeding tokens from multiple passages into LLMs. Furthermore, the training time when $n$ increases is shown in Table~\ref{tab:runtime-imbd}, and we can see that the runtime for training is typically no more than one minute.

\begin{table}[H]
\centering
\caption{Runtime scale with $d$.}\label{tab:runtime-d}
\begin{tabular}{lccccc}
\toprule
$d$ & 4 & 8 & 12 & 16 & 20 \\
\midrule
Runtime & 44.48 & 44.62 & 44.73 & 44.92 & 45.00 \\
\bottomrule
\end{tabular}
\end{table}
\begin{table}[H]
\centering
\caption{Runtime (memory in parenthesis) scale with $n$. The runtime is measured in seconds and the memory is measured in GB. }\label{tab:runtime-imbd}
\setlength{\tabcolsep}{3pt}
\begin{tabular}{lcccccc}
\toprule
$n$ & 100 & 150 & 200 & 250 & 300 & 350 \\
\midrule
% ImBD & 45.65(9.56) & 68.20(9.57) & 93.00(9.57) & 116.08(9.56) & 136.53(9.56) & 158.15(9.58) \\
Runtime (Memory) & 9.28(0.36) & 23.45(0.37) & 40.19(0.37) & 44.25(0.37) & 59.57(0.60) & 69.56(0.37) \\
\bottomrule
\end{tabular}
\end{table}

\subsection{Sensitivity analysis}\label{sec:tuning}

Since AdaDetectGPT requires training a witness function, we examine three factors influencing its performance: (1) the size of the training set; (ii) tuning parameters for generating B-spline basis and (iii) distribution shift between training and test data. 

\textbf{Robust to training data sizes}. We evaluate AdaDetectGPT across varying dataset sizes by setting $n_1 = n_2 \in \{100, 200, 300, 400, 500, 600\}$ for both human- and machine-generated texts. 
Figure~\ref{fig:sample} demonstrates that AdaDetectGPT clearly outperforms FastDetectGPT when sample size is large. This is expected because a larger sample size leads to a more accurate estimation of $w$. Notably, even with limited data $n_1=n_2=100$, AdaDetectGPT maintains superior accuracy compared to baseline methods, though the performance gap decreases. These results highlight our method's effectiveness on learning the witness function. 

\textbf{Insensitivity to tuning parameters}. B-spline relies on two critical tuning parameters: (i) the number of basis functions (\texttt{n\_base}) and (ii) the maximum polynomial order. Our experiments fix one parameter while varying the other (with \texttt{n\_base}=16 or order=2 as defaults). As shown in Figure~\ref{fig:tuning} in Appendix, AdaDetectGPT achieves the highest AUC scores so long as \texttt{n\_base} $\geq 4$. Besides, enlarging \texttt{n\_base} improves the AUC of AdaDetectGPT although the improvement becomes marginal when \texttt{n\_base} $\geq 16$. Figure~\ref{fig:tuning} also shows that increasing the polynomial order from linear to quadratic visibly improve the performance; while increasing order from quadratic to cubic/quartic has a limited gain. Finally, even when the B-spline basis is set to a piecewise linear function, our method still outperform all baselines. 

\textbf{Robust against distribution shift}. We create training datasets with different distributions from the test data by varying the number of human prompt tokens in machine-generated text. In contrast, for the test data, the number of human prompt tokens are fixed. As shown in Figure~\ref{fig:dist-shift}, AdaDetectGPT demonstrates high robustness to the distributional discrepancy between training and test data. It achieves the highest AUC across all experimental setup.

\begin{figure}[htbp]
    \centering
    \includegraphics[width=1.0\linewidth]{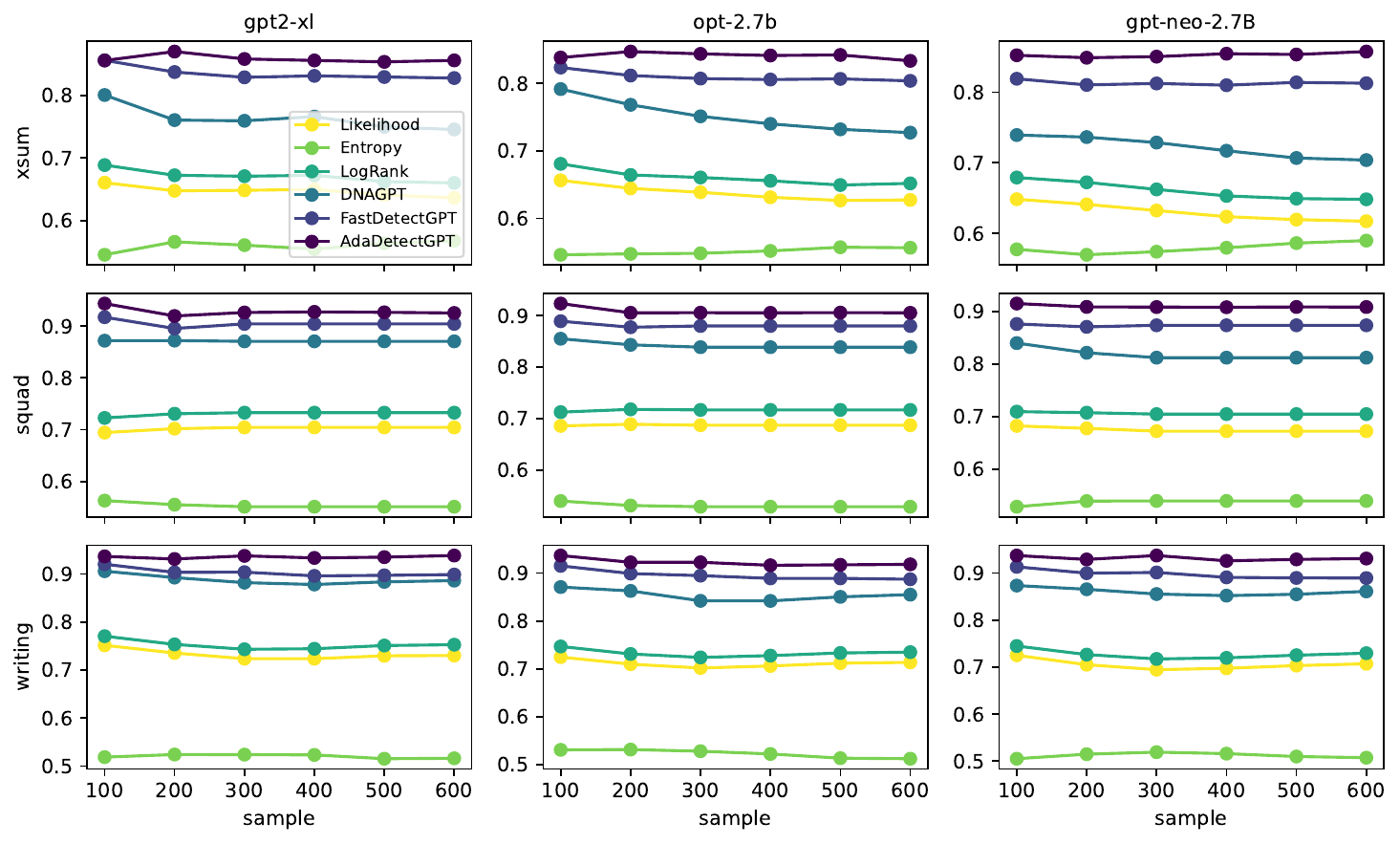}
    \vspace*{-20pt}
    \caption{Classification accuracy versus the sample size for training $w$. We omit DetectGPT, NPR, and DNA in this experiments as they are time-consuming.}
    \label{fig:sample}
\end{figure}

\begin{figure}
    \centering
    \includegraphics[width=1.0\linewidth]{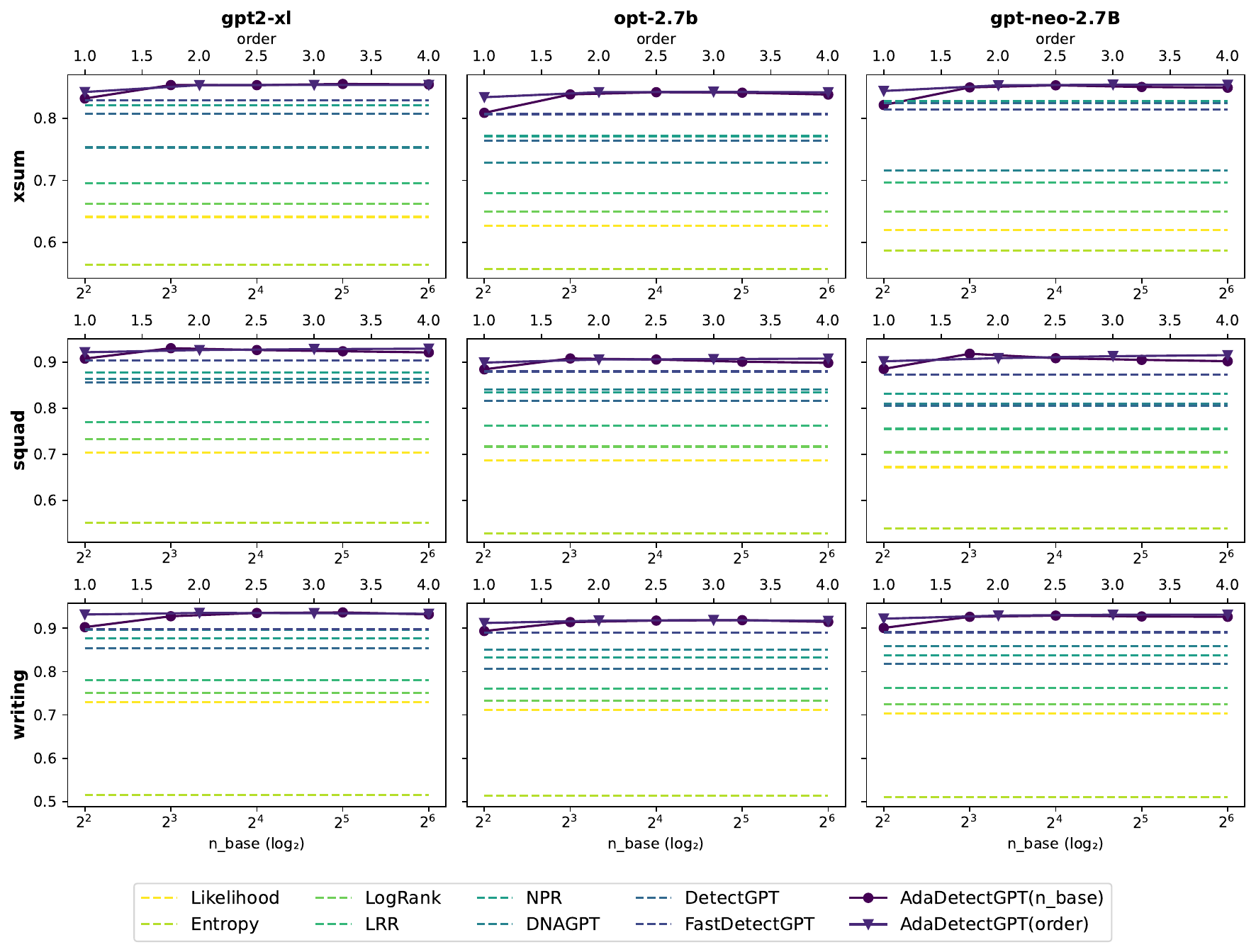}
    \vspace*{-20pt}
    \caption{The classification accuracy of AdaDetectGPT and baseline methods. AdaDetectGPT(\texttt{n\_base}) present the AUC when the number of basis in B-spline increases as 4, 8, 16, 32, 64 (bottom $x$-axis); while AdaDetectGPT(\texttt{order}) shows the AUC when the maximum order of basis in B-spline increases from 1 to 4 (top $x$-axis). The AUC of baseline methods are presented by dash lines. }
    \label{fig:tuning}
\end{figure}

\begin{figure}
    \centering
    \includegraphics[width=1.0\linewidth]{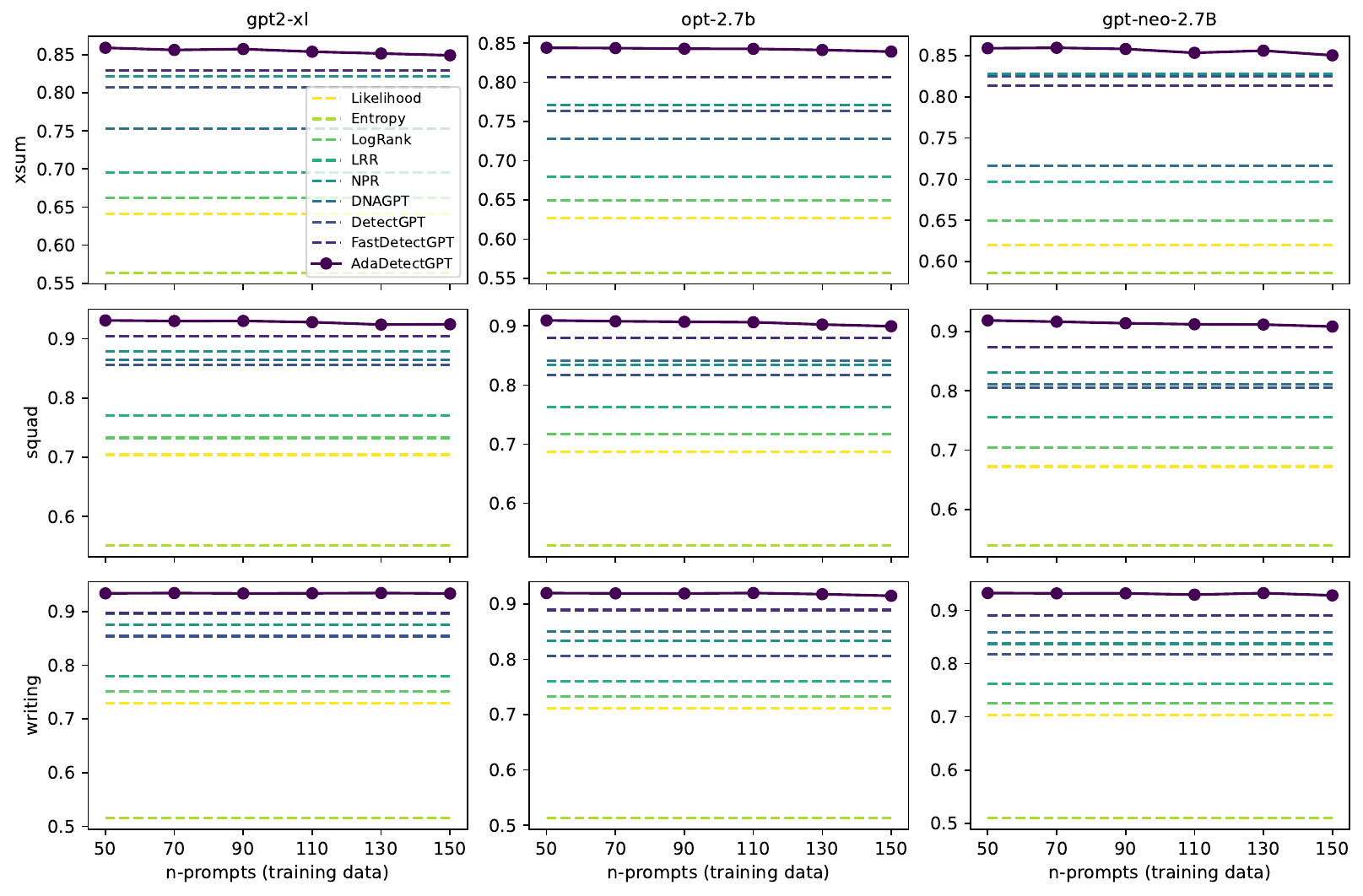}
    \vspace*{-20pt}
    \caption{The classification accuracy of AdaDetectGPT when the number of human prompts changes. The AUC of baseline methods are presented by dash lines. }
    \label{fig:dist-shift}
\end{figure}

\subsection{Detecting open-source models in black-box setting}\label{sec:addblackbox}

\begin{table}[H]
    \caption{Zero-shot detection accuracy on five source models under the black-box setting. $^\dagger$: use two surrogate models for sampling and scoring, where the sampling model is GPT-J while the scoring model is GPT-Neo.}\label{tab:main_results_blackbox_details}
    \centering\small
    \begin{tabular}{llccccc}
        \toprule
        \multirow{2}{*}{\bf Dataset} & \multirow{2}{*}{\bf Method} & \multicolumn{5}{c}{\bf Source Model} \\
         &  & GPT-2 & OPT-2.7 & GPT-Neo & GPT-NeoX & \bf Avg. \\   
        \midrule
        \multirow{6}{*}{SQuAD} 
        & FastDetectGPT & 0.6181 & 0.6495 & 0.6230 & 0.6910 & 0.6813 \\
        & AdaDetectGPT & 0.6920 & 0.7195 & 0.7382 & 0.7338 & 0.7460 \\
        \cdashline{2-7}
        & Relative & 19.3570 & 19.9651 & 30.5609 & 13.8495 & 20.2957 \\
        \cline{2-7}
        & FastDetectGPT$^\dagger$ & 0.8145 & 0.8166 & 0.9220 & 0.7519 & 0.8188 \\
        & AdaDetectGPT$^\dagger$ & \textbf{0.8249} & \textbf{0.8308} & \textbf{0.9273} & \textbf{0.7609} & \textbf{0.8300} \\
        \cdashline{2-7}
        & Relative & 5.6301 & 7.7245 & 6.7968 & 3.6121 & 6.2106 \\
        \midrule
        \multirow{6}{*}{Writing} 
        & FastDetectGPT & 0.7662 & 0.7918 & 0.7685 & 0.8022 & 0.8028 \\
        & AdaDetectGPT & 0.8306 & 0.8529 & 0.8555 & \textbf{0.8587} & 0.8636 \\
        \cdashline{2-7}
        & Relative & 27.5699 & 29.3365 & 37.6112 & 28.5350 & 30.8124 \\
        \cline{2-7}
        & FastDetectGPT$^\dagger$ & 0.8565 & 0.8497 & 0.9215 & 0.8182 & 0.8582 \\
        & AdaDetectGPT$^\dagger$ & \textbf{0.8780} & \textbf{0.8737} & \textbf{0.9386} & 0.8567 & \textbf{0.8849} \\
        \cdashline{2-7}
        & Relative & 14.9666 & 15.9741 & 21.7742 & 21.1856 & 18.8023 \\        
        \midrule
        \multirow{6}{*}{XSum} 
        & FastDetectGPT & 0.5919 & 0.6445 & 0.5718 & 0.6389 & 0.6468 \\
        & AdaDetectGPT & 0.6795 & 0.7238 & 0.6879 & 0.7045 & 0.7261 \\
        \cdashline{2-7}
        & Relative & 21.4569 & 22.2991 & 27.1129 & 18.1580 & 22.4439 \\
        \cline{2-7}
        & FastDetectGPT$^\dagger$ & 0.8145 & 0.8166 & 0.9220 & 0.7519 & 0.8188 \\
        & AdaDetectGPT$^\dagger$ & \textbf{0.8249} & \textbf{0.8308} & \textbf{0.9273} & \textbf{0.7609} & \textbf{0.8300} \\
        \cdashline{2-7}
        & Relative & 9.8060 & 10.5637 & 10.2543 & 8.1057 & 11.1574 \\
        \bottomrule
    \end{tabular}
\end{table}

\section{Broader impact and limitation}\label{sec:limitation}
AdaDetectGPT is a computationally and statistically efficient detector for machine-generated text, thus safeguarding AI systems against fake news, disinformation, and academic plagiarism. 

Despite AdaDetectGPT's strong empirical performance in the black-box setting, its theoretical guarantees are mainly established in the white-box setting. Even when restricting to the white-box setting, LLM text generation often involves sampling parameters (e.g., \texttt{temperature} and \texttt{top\_k}). Using different parameter values can cause the sampling distribution to deviate from that of the target model we aim to detect. This mismatch invalidates MCLT in practice. Fortunately, we observe that the shape of our statistic remains similar, but shifts toward a positive mean (see Figure~\ref{fig:normal-under-h1}), implying that FNR control under MCLT remains valid, although being more conservative.

\begin{figure}[H]
    \centering
    \includegraphics[width=1.0\linewidth]{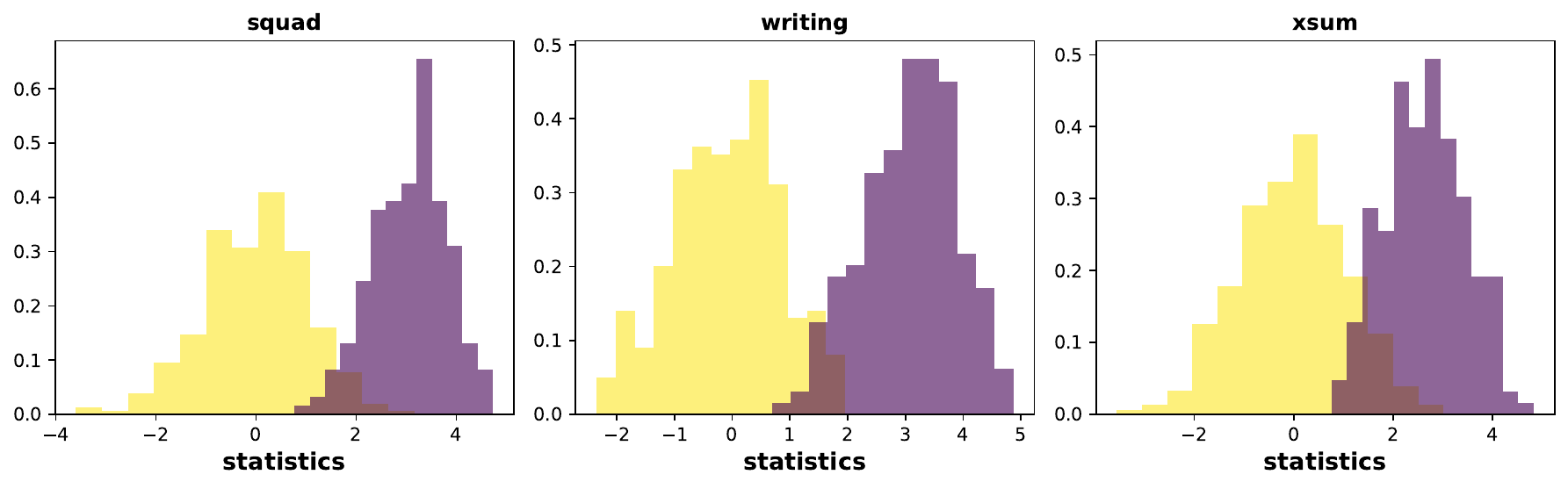}
    \vspace*{-18pt}
    \caption{Histogram of our statistics in three datasets. Each panel visualizes the histogram in one dataset. The yellow histogram corresponds to the case when the sampled texts exactly follow the conditional probability of the source model, while purple histogram corresponds to text drawn with from a distribution with different sampling temperatures.}
    \label{fig:normal-under-h1}
\end{figure}

% Our method design leans on pre-trained models to span a multitude of domains and languages. This presents a challenge in a black-box setting, as no single model can seamlessly span all linguistic territories and domains. This is due to the intrinsic nature of pre-trained models being tailored to specific domains and languages.

\end{document}